%% file: paper.tex
\documentclass{article}
\usepackage{kr}

\usepackage{times}
\usepackage{soul}
\usepackage{url}
\usepackage[hidelinks]{hyperref}
\usepackage[utf8]{inputenc}
\usepackage[small]{caption}
\usepackage{graphicx}
\usepackage{amsmath}
\usepackage{amsthm}
\usepackage{multirow}
\usepackage{booktabs}
\usepackage{diagbox}
\usepackage{algorithm}
\usepackage{algorithmic}
\usepackage{graphicx}
\usepackage{microtype}
\usepackage{multirow}
\usepackage[dvipsnames]{xcolor}
\usepackage{etoolbox}
\usepackage{amsmath}
\usepackage{latexsym,amssymb}
\usepackage{upgreek}
\usepackage{array}
\usepackage{adjustbox}
\usepackage{cite}
\usepackage{soul}
\usepackage{url}
\usepackage[utf8]{inputenc}
\usepackage{caption} %[small]
\usepackage{booktabs}
\usepackage{colortbl}
\usepackage{empheq}
\usepackage{fancybox}
\usepackage{subcaption}
\usepackage[english]{babel}
\usepackage{xspace}
\usepackage{setspace}
\usepackage{stmaryrd}
\usepackage{enumitem}
\usepackage{ifthen}
\usepackage{verbatim}
\usepackage{titlesec}
\usepackage{amsthm}
\usepackage[capitalise,nameinlink]{cleveref}
\usepackage[compress]{natbib}

\usepackage{xfrac}
\usepackage[framemethod=TikZ]{mdframed}
\usepackage{nicefrac}       % compact symbols for 1/2, etc.
\usepackage{pifont}

\usepackage{tikz}
\usepackage{forest}

\usetikzlibrary{arrows,decorations.pathmorphing,decorations.footprints,fadings,calc,trees,mindmap,shadows,decorations.text,patterns,positioning,shapes,matrix,fit}
\usetikzlibrary{arrows,shadows,backgrounds}
\usetikzlibrary{arrows.meta}
\usetikzlibrary{positioning,fit}
\usetikzlibrary{automata}
\usetikzlibrary{matrix}
\usetikzlibrary{shapes.symbols,shapes.misc,shapes.arrows,shapes.geometric}
\usetikzlibrary{matrix,chains,scopes,decorations.pathmorphing}

\newtheorem{example}{Example}

\input{defs}

\input{preamble}

\begin{document}

\maketitle

\input{abs}
\input{intro}

\input{prelim}

\input{relw}

\input{formpi}

\input{nprop}
%%\input{./tabs/sccmp} % Temporary!!!
\input{nidx}

\input{tres}
%% Other files...
\input{conc}

\input{appendix}

\input{replbib}
\input{togbbl} % file is automatically generated

% ---- Bibliography ----
%%\cleardoublepage %% TENTATIVE, and required if bibliography starts page...
%%\addcontentsline{toc}{section}{References}
%%\vskip 0.2in
% For arxix paper production, and since arXiv does not allow for
% bibtex, we need to create a .bbl file to include upon submission
% to arXiv.
\iftoggle{mkbbl}{
  % Run bibtex, i.e. generate .bbl gile
    \bibliographystyle{kr}
    \bibliography{refs,xtra,team}
}{
  % Import bibl (original .bbl) file

\input{paper.bibl}
}

\end{document}

%% file: defs.tex
\newboolean{nonanon}               % Anonymous submission
\setboolean{nonanon}{true}         %
\newboolean{mkcompact}             % Compact submission
\setboolean{mkcompact}{true}       %
\newboolean{squeezeit}             % Apply baselinestretch
\setboolean{squeezeit}{true}       %
\hypersetup{
  colorlinks = true,
  linkcolor = MidnightBlue,%darkblue,
  urlcolor  = MidnightBlue,%darkblue,
  citecolor = MidnightBlue,%darkblue,
  anchorcolor = MidnightBlue%darkblue
}

\crefname{paragraph}{paragraph \nameref}{paragraphs \nameref}
\Crefname{paragraph}{Paragraph \nameref}{Paragraphs \nameref}

\newtheorem{proposition}{Proposition}

\crefname{prop}{Proposition}{Propositions}
\Crefname{prop}{Proposition}{Propositions}
\crefname{defn}{Definition}{Definitions}
\Crefname{defn}{Definition}{Definitions}
\crefname{cor}{Corollary}{Corollaries}
\Crefname{cor}{Corollary}{Corollaries}
\crefname{exmpl}{Example}{Examples}
\Crefname{exmpl}{Example}{Examples}
\Crefname{theorem}{Theorem}{Theorem}

\newtheoremstyle{nutstyle}
{3.0pt} % Space above
{3.0pt} % Space below
{} % Body font
{} % Indent amount
{\bfseries} % Theorem head font
{.} % Punctuation after theorem head
{.5em} % Space after theorem head
{} % Theorem head spec (can be left empty, meaning `normal')

\theoremstyle{nutstyle}

\newtheoremstyle{nuxstyle}
{3.0pt} % Space above
{3.0pt} % Space below
{} % Body font
{} % Indent amount
{\itshape} % Theorem head font
{.} % Punctuation after theorem head
{.5em} % Space after theorem head
{} % Theorem head spec (can be left empty, meaning `normal')

\theoremstyle{nuxstyle}

\newcommand{\todoF}[2]{}

\newcommand{\fml}[1]{{\mathcal{#1}}}

\newcommand{\tn}[1]{\textnormal{#1}}
\newcommand{\mbf}[1]{\ensuremath\mathbf{#1}}
\newcommand{\msf}[1]{\ensuremath\mathsf{#1}}
\newcommand{\mbb}[1]{\ensuremath\mathbb{#1}}

\newcommand{\relevant}{\oper{Relevant}}
\newcommand{\irrelevant}{\oper{Irrelevant}}

\newcommand{\exv}{\ensuremath\mathbf{E}}

\newcommand{\prob}{\ensuremath\mathbf{P}}

\newcommand{\cf}{\ensuremath\upsilon} % \nu, \xi
\newcommand{\cfn}[1]{\ensuremath\upsilon_{#1}} % \;\!\! \nu, \xi
\newcommand{\cfd}[1]{\ensuremath\upsilon^{d}_{#1}} % \;\!\! \nu, \xi

\DeclareMathOperator*{\sv}{\msf{Sc}}
\newcommand{\svn}[1]{\msf{Sc}_{#1}}
\newcommand{\svd}[1]{\msf{Sc}^{d}_{#1}}

\newcommand{\tsv}{\msf{TSc}}
\newcommand{\tsvn}[1]{\msf{TSc}_{#1}}

\newcommand{\critic}{\ensuremath\mathsf{Crit}}

\newcommand{\pidx}{\msf{Idx}}

\newcommand{\piterms}{feature importance scores\xspace}
\newcommand{\pitermscap}{Feature Importance Scores\xspace}
\newcommand{\piacronympl}{FISs\xspace}
\newcommand{\piacronym}{FIS\xspace}

\definecolor{darkslategray}{rgb}{0.18, 0.31, 0.31} %#2F4F4F
\definecolor{platinum}{rgb}{0.9, 0.89, 0.89} %#E5E4E2
\definecolor{gray}{rgb}{.4,.4,.4}
\definecolor{midgrey}{rgb}{0.5,0.5,0.5}
\definecolor{middarkgrey}{rgb}{0.35,0.35,0.35}
\definecolor{darkgrey}{rgb}{0.3,0.3,0.3}
\definecolor{darkred}{rgb}{0.7,0.1,0.1}
\definecolor{midblue}{rgb}{0.2,0.2,0.7}
\definecolor{darkblue}{rgb}{0.1,0.1,0.5}
\definecolor{defseagreen}{cmyk}{0.69,0,0.50,0}

\newcommand{\jnoteF}[1]{}
\newcommand{\anoteF}[1]{}

\newcommand{\oper}[1]{\ensuremath\textnormal{\small{\textsf{#1}}}}

\newcounter{tableeqn}[table]

\DeclareMathOperator*{\lequiv}{\leftrightarrow}
\DeclareMathOperator*{\limply}{\rightarrow}

\newcommand{\waxp}{\ensuremath\mathsf{WAXp}}
\newcommand{\wcxp}{\ensuremath\mathsf{WCXp}}
\newcommand{\axp}{\ensuremath\mathsf{AXp}}
\newcommand{\cxp}{\ensuremath\mathsf{CXp}}

\newcommand{\wc}{\ensuremath\mathsf{WC}}

\newcommand{\mailtodomain}[1]{\href{mailto:#1@ciencias.ulisboa.pt}{\texttt{\nolinkurl{#1}}}}

\titleformat{\paragraph}[runin]
{\normalfont\bfseries}{}{0pt}{}

\newcolumntype{L}[1]{>{\raggedright\let\newline\\\arraybackslash\hspace{0pt}}m{#1}}
\newcolumntype{C}[1]{>{\centering\let\newline\\\arraybackslash\hspace{0pt}}m{#1}}
\newcolumntype{R}[1]{>{\raggedleft\let\newline\\\arraybackslash\hspace{0pt}}m{#1}}

\titlespacing{\paragraph}{%
  0pt}{%              left margin
  0.175\baselineskip}{% space before (vertical)
  1em}%               space after (horizontal)

\newcommand{\cmark}{\mbox{\ding{51}}}
\newcommand{\xmark}{\mbox{\ding{55}}}

%% file: preamble.tex
\title{From SHAP Scores to Feature Importance Scores}

\iftrue
\author{%
Olivier L\'{e}toff\'{e}$^{1}$\and
Xuanxiang Huang$^{2}$\and
Nicholas Asher$^{3}$\And
Joao Marques-Silva$^{4}$\\
\affiliations
$^1$IRIT, University of Toulouse, France\and
$^2$CNRS@CREATE, Singapore\\
$^3$IRIT, CNRS, France\and
$^4$ICREA, University of Lleida, Spain\\
\emails
olivier.letoffe@univ-toulouse.fr,
xuanxiang.huang@cnrsatcreate.sg,
asher@irit.fr,
jpms@icrea.cat
}
\fi

%% file: abs.tex
%% Measures of inconsistency...
% 
\begin{abstract}
  %One overarching
  A central goal of eXplainable Artificial Intelligence (XAI) is
  to assign relative importance to the features of a Machine Learning
  (ML) model given some prediction.
  The importance of this task of explainability by feature attribution
  is illustrated by the ubiquitous recent use of tools such as SHAP
  and LIME.
  %
  %This task is referred to as explainability by feature attribution,
  %and its importance is illustrated by the ubiquitous recent % widespread
  %use of tools such as SHAP and LIME.
  %
  Unfortunately, the exact computation of feature attributions, using
  the game-theoretical foundation underlying SHAP and LIME, can yield
  manifestly unsatisfactory results, that tantamount to reporting %producing
  misleading relative feature importance. %information regarding 
  Recent work targeted rigorous feature attribution, by studying
  axiomatic aggregations of features based on logic-based definitions
  of explanations by feature selection.
  This paper shows that there is an essential relationship between
  feature attribution and a priori voting power, and that those
  recently proposed axiomatic aggregations represent a few
  instantiations of the range of power indices studied in the past.
  %
  %Recent work has investigated the use of power indices, a well-known
  %measure of a priori voting power
  %%also with links to game theory,
  %as a possible solution for rigorous feature attribution.
  %%
  %These proposed uses of power indices build on logic-based
  %definitions of explanations by feature selection, which are commonly
  %referred to as abductive explanations.
  %
  Furthermore, it remains unclear how some of the most widely used
  power indices might be exploited as \piterms (\piacronympl), i.e.\
  the use of power indices in XAI, and which of these indices would be
  the best suited for the purposes of XAI by feature attribution,
  namely in terms of not producing results that could be deemed as
  unsatisfactory.
  This paper proposes novel desirable properties that \piacronympl
  should exhibit.
  In addition, the paper also proposes novel \piacronympl exhibiting
  the proposed properties.
  Finally, %Furthermore,
  the paper conducts %gives
  a rigorous analysis of the best-known power indices in terms of the
  proposed properties.
  %, and clarifies a few existing misconceptions.
  %
\end{abstract}

%% file: intro.tex
\section{Introduction}
\label{sec:intro}

The requirement %importance
of delivering trustworthy systems of
artificial intelligence (AI) and
machine learning (ML)
is arguably one of the most consequential challenges in current
computing research and practice.
A growing number of governments have put forth proposals for regulating
the use of systems of AI/ML~(\citealp{wheo23}; \citealp{euaict24}),
some of which are in the process of becoming enforced
regulations. Many companies that are stakeholders in AI/ML are
actively involved in building trust into the operation of highly
complex systems of AI/ML. %~\citep{}
Furthermore, it is widely accepted that trustworthy AI/ML can only be
achieved through the comprehensive use of (trustable) methods for
explaining the operation of those complex systems of
AI/ML~\citep{alegai19}.
This critical importance of explainability justifies in part the
massive interest in eXplainable AI (XAI) in recent
years~(\citealp{muller-dsp18}; \citealp{pedreschi-acmcs19};
\citealp{xai-bk19}; \citealp{muller-ieee-proc21}).
Furthermore, in domains where the uses of AI/ML are deemed of
high-risk or are safety-critical, XAI must ensure the rigor of
computed explanations.
Unfortunately, the most widely-used methods of XAI offer no guarantees
of rigor. Unsurprisingly, several results in recent years have
unveiled critical misconceptions of informal
XAI~(\citealp{ignatiev-ijcai20}; \citealp{ms-iceccs23}).

Motivated by this state of affairs, 
a formal XAI~(\citealp{msi-aaai22}; \citealp{ms-rw22};
\citealp{darwiche-lics23}) that builds on logic-based  
definitions of explanations and that in turn is deployed using
highly efficient automated reasoners has emerged.
However, formal XAI also exhibits some drawbacks, including the
complexity of automated reasoning, but also the need to offer human
decision-makers mechanisms for summarizing information about
explanations. While non-rigorous methods of XAI have used so-called
SHAP scores~\citep{lundberg-nips17} as a mechanism to convey 
to human decision-makers information about relative feature
importance, formal XAI focused instead on finding explanations based
on feature selection, i.e.\ finding sets of features, with no apparent
solution for addressing relative feature importance.

SHAP scores have become a highly popular method in
XAI~(\citealp{molnar-bk23}; \citealp{groh-coling22};
\citealp{mishra-bk23}). 
Nevertheless, recent work has demonstrated that the
rigorous computation of SHAP scores can be thoroughly
unsatisfactory~(\citealp{hms-corr23a}; \citealp{hms-corr23d};
\citealp{hms-ijar24}; \citealp{msh-cacm24}).
Indeed, it is now known that even for very simple classifiers, the
computed SHAP scores can bear little relationship with self-evident
notions of feature importance. An immediate consequence is that
computed SHAP scores can be misleading, i.e.\ they can induce human
decision-makers in error, by assigning more importance to the wrong
features.
Motivated by these negative results, researchers have proposed alternative
solutions for assessing relative feature
importance~(\citealp{ignatiev-corr23a}; \citealp{izza-corr23};
\citealp{ignatiev-corr23b}; \citealp{izza-aaai24}), all of
which are based on logic-based definitions of explanations.
Furthermore, axiomatic aggregations~\citep{izza-aaai24}
%, based on logic-based definitions of explanations,
also reveal that feature importance can be naturally related with
power indices, i.e.\ measures of voter importance in a priori voting
power~(see~\citep{machover-hscv15} for an overview), but also with
payoff vectors, i.e.\ measures of player rewards in characteristic
function games~(see~\citep{elkind-bk12}).
%
%
%Furthermore, recent work~\citep{izza-aaai24} showed that logic-based
%explanations can serve for defining rigorous measures of feature
%attribution, which in turn are naturally related with power indices,
%i.e.\ measures of voter importance in a priori voting
%power~(see~\citep{machover-hscv15} for an overview), but also payoff
%vectors, i.e.\ measures of player rewards in characteristic function
%games~(see~\citep{elkind-bk12}).
%

It should be underscored that there is a crucial difference between
the focus of power indices for measuring a priori voting
power~\citep{machover-hscv15} and the focus of SHAP
scores~\citep{lundberg-nips17}.
In power indices, voter importance measures how likely each voter is
to be \emph{critical} for a set of voters to represent a winning
coalition (e.g.\ see~\citep{shapley-apsr54}).
In contrast, the existing definitions of SHAP scores \emph{overlook}
the criticality of features, and instead measure feature contributions
starting from the expected values of the ML model.
As discussed in this paper, this fundamental difference is at the core
of why SHAP scores can yield unsatisfactory results.

%%%Perhaps more importantly, t
%%It should be underscored that there is a critical difference between
%%recent work on axiomatic aggregations~\citep{izza-aaai24} and past
%%work on SHAP scores~\citep{lundberg-nips17}.
%%%
%%%%More importantly, i
%%In recent work on axiomatic aggregations, which mimics earlier work on
%%power indices and also on payoff vectors,
%%%a priori voting power and also on characteristic function games,
%%voter importance measures how likely each voter is to be critical for
%%a set of voters to represent a winning coalition
%%(e.g.\ see~\citep{shapley-apsr54}). In contrast, existing proposals of
%%SHAP scores \emph{overlook} the criticality of features, and instead
%%measure feature contributions based on the expected values of the ML
%%model.
%%%
%%As discussed in this paper, this fundamental difference is at the core
%%for why SHAP scores can yield unsatisfactory results.
%%

Furthermore, and although the connection between feature attribution
and power indices proposed in recent work~\citep{izza-aaai24} is
intuitive, it is unclear which power indices are best suited for
feature attribution, or even which %what 
properties such indices %for feature attribution
should exhibit.  Our paper addresses these issues.  
%\paragraph{Our contributions.}
%~\\
%
In particular, our work extends recent work on axiomatic
aggregations~\citep{izza-aaai24} and their relationship with formal
explanations in several
ways.
First, we uncover the tight connection between feature attribution in
XAI and power indices in {\em a priori} voting power.
Second, we argue that feature attribution in XAI requires a more
general formalization than power indices, and develop such a
formalization, namely \piterms (\piacronympl).
Third, we overview several properties of interest for \piacronympl,
and propose novel properties.
Fourth, and motivated by the results in this paper, we propose novel
\piacronympl, which can also find application in other domains.
Fifth, and finally, we provide a detailed characterization of which 
\piacronympl exhibit which properties.

%% file: prelim.tex
\section{Preliminaries}
\label{sec:prelim}

\jnoteF{%
  Classifiers.\\
  Logic-Based Explanations.\\
  SHAP Scores.\\
  Power Indices.\\
}

\paragraph{Classification problems.}
%~\\
%
Let $\fml{F}=\{1,\ldots,m\}$ denote a set of features and
$\fml{K}=\{c_1,c_2,\ldots,c_K\}$ denote a set of classes.
Each feature $i\in\fml{F}$ takes values from a domain $\mbb{D}_i$ which is either categorical or ordinal. If ordinal, a domains can be
discrete or real-valued.
Classes can also be categorical or ordinal.
Throughout the paper, we assume that domains are discrete-valued and that,
unless otherwise stated, classes are  ordinal.

A feature space is defined by
$\mbb{F}=\mbb{D}_1\times\mbb{D}_2\times\ldots\times\mbb{D}_m$. 
The notation $\mbf{x}=(x_1,\ldots,x_m)$ denotes an arbitrary point in 
feature space, where each $x_i$ is a variable taking values from
$\mbb{D}_i$. Moreover, the notation $\mbf{v}=(v_1,\ldots,v_m)$
represents a specific point in feature space, where each $v_i,
i=1,\ldots,m$ is a constant representing one concrete value from
$\mbb{D}_i$.
An \emph{instance} denotes a pair $(\mbf{v}, c)$, where
$\mbf{v}\in\mbb{F}$ and $c\in\fml{K}$, and such that
$c=\kappa(\mbf{v})$.
An ML classifier $\fml{M}$ is characterized by a
non-constant \emph{classification function} $\kappa$ that maps feature
space $\mbb{F}$ into the set of classes $\fml{K}$,
i.e.\ $\kappa:\mbb{F}\to\fml{K}$.
Given the above, we associate with a classifier $\fml{M}$, a tuple 
$(\fml{F},\mbb{F},\fml{K},\kappa)$.%
\footnote{The paper considers only classification problems. However,
under a suitable formalization, the paper's results are also valid for
regression problems.}

\begin{comment}
%
If both $\mbb{D}_i=\{0,1\}, i=1,\ldots,m$ and $\fml{K}=\{0,1\}$, then
the classifier is referred to as a \emph{boolean}, in which case we
use $\mbb{B}=\{0,1\}$.
%
If the set of classes is ordinal but non-boolean, then the classifier
is referred to as \emph{multi-valued}.
%
Finally, if both the domains and the set of classes are ordinal (and
discrete), then the classifier is referred to as \emph{discrete}.
%
\end{comment}

\paragraph{Selection of sets of points.}
%~\\
Given $\mbf{x}$ and $\mbf{v}$, and a set $\fml{S}\subseteq\fml{F}$, we
define the predicate $\mbf{x}_{\fml{S}}=\mbf{v}_{\fml{S}}$ to
represent the logic statement $(\land_{i\in\fml{S}}x_i=v_i)$.
Furthermore $\Upsilon:2^{\fml{S}}\to2^{\mbb{F}}$, is defined by
$\Upsilon(\fml{S};\mbf{v})=\{\mbf{x}\in\mbb{F}\,|\,\mbf{x}_{\fml{S}}=\mbf{v}_{\fml{S}}\}$. Thus,
$\mbf{x}_{\fml{S}}=\mbf{v}_{\fml{S}}$ iff
$\mbf{x}\in\Upsilon(\fml{S};\mbf{v})$.

Similarly, given an instance $(\mbf{v},c)$, and for $\mbf{x}\in\mbb{F}$,
$\fml{I}(\mbf{x};\mbf{v}) = \{i\in\fml{F}\,|\,x_i=v_i\}$ denotes the
features for which both $\mbf{x}$ and $\mbf{v}$ take the same values.

%
\begin{comment}
%
Throughout the paper, it will often be necessary to represent sets of
points in feature space that are consistent with some other point in
feature space with respect to the features dictated by some set of
features.
%
Accordingly, we define $\Upsilon:2^{\fml{F}}\to2^{\mbb{F}}$ as
follows,%
\footnote{% 
Parameterizations are shown as arguments after the separator ';', and
their purpose is to keep the notation simple.
%
In addition, and also for simplicity, we will elide parameterizations
whenever these are clear from the context.}
%
\begin{equation} \label{eq:upsilon}
  \Upsilon(\fml{S};\mbf{v}) := \{\mbf{x}\in\mbb{F}\,|\,\land_{i\in\fml{S}}x_i=v_i\}
\end{equation}
%
i.e.\ for some $\fml{S}\subseteq\fml{F}$, and parameterized by
the point $\mbf{v}$ in feature space, $\Upsilon(\fml{S};\mbf{v})$
denotes all the points $\mbf{x}=(x_1,\ldots,x_m)\in\mbb{F}$ in feature
space that have in common with $\mbf{v}=(v_1,\ldots,v_m)\in\mbb{F}$
the values of the features specified by $\fml{S}$.
%
%Observe that $\Upsilon$ is also used (implicitly) for picking the set
%of rows we are interested in when computing explanations.
%
Finally, we write $\mbf{x}_{\fml{S}}=\mbf{v}_{\fml{S}}$ to signify
that $\mbf{x}\in\Upsilon(\fml{S};\mbf{v})$.
%
\end{comment}

\paragraph{Expected value.}
The \emph{expected value} of a classification function $\kappa$ is
denoted by $\mbf{E}[\kappa]$.
%For a complete data point $\mbf{v}$, we have
%$\mbf{E}[\kappa|_{\mbf{v}}] = \kappa(\mbf{v})$.
%
Furthermore, 
$\exv[\kappa\,|\,\mbf{x}_{\fml{S}}=\mbf{v}_{\fml{S}}]$ represents the
expected value of $\kappa$ over points in feature space consistent with the
coordinates of $\mbf{v}$ dictated by $\fml{S}$, which is defined as
follows:
%
%$\mbf{E}[\kappa|_{\mbf{v}_{\fml{S}}}]$ denote the expected value of
%the boolean function $\kappa|_{\mbf{v}_{\fml{S}}}$, which is defined
%as follow:
%
\begin{equation}
  %\mbf{E}[\kappa|_{\mbf{v}_{\fml{S}}}]:=\sum\nolimits_{\mbf{x}\in\Upsilon(\fml{S};\mbf{v})}\kappa(\mbf{x})\cdot P(\mbf{x}|\mbf{v}_{\fml{S}})
  \exv[\kappa\,|\,\mbf{x}_{\fml{S}}=\mbf{v}_{\fml{S}}]
  :=\sfrac{1}{|\Upsilon(\fml{S};\mbf{v})|}
  \sum\nolimits_{\mbf{x}\in\Upsilon(\fml{S};\mbf{v})}\kappa(\mbf{x})
\end{equation}

\paragraph{Explanation problems.}
%~\\
Given a classification problem $\fml{M}$ and a concrete instance
$(\mbf{v},c)$, an \emph{explanation problem} $\fml{E}$ is a tuple
$(\fml{M},(\mbf{v},c))$. When describing concepts in explainability,
we assume an underlying explanation problem $\fml{E}$, with
all definitions parameterized on $\fml{E}$.

\paragraph{Shapley values.}
%~\\
%
Shapley values were proposed in the context of game theory %in the
%early 1950s 
by L.\ S.\ Shapley~\citep{shapley-ctg53}. Shapley values
are defined relative to a set $\fml{R}$, and a \emph{characteristic
function}, i.e.\ a real-valued function defined on the subsets of
$\fml{R}$, $\cf:2^{\fml{R}}\to\mbb{R}$, 
such that $\cf(\emptyset)=0$.%
\footnote{%
The original formulation also required super-additivity of the
characteristic function, but that condition has been relaxed in more
recent works~(\citealp{dubey-ijgt75}; \citealp{young-ijgt85}).}
It is well-known that Shapley values represent the \emph{unique}
function that, given $\fml{R}$ and $\cf$, respects a number of
important axioms (or properties). More detail about Shapley values is
available in standard
references~(\citealp{shapley-ctg53}; \citealp{dubey-ijgt75};
\citealp{young-ijgt85}; \citealp{roth-bk88}; \citealp{elkind-bk12}).

\paragraph{SHAP scores.}
%~\\
%
In the context of explainability, Shapley values are most often
referred to as SHAP scores%
%, i.e. the application of Shapley values in 
%explainability
~(\citealp{kononenko-jmlr10}; \citealp{kononenko-kis14};
\citealp{lundberg-nips17}; \citealp{barcelo-aaai21};
\citealp{barcelo-jmlr23}),
and consider a specific characteristic function
$\cfn{e}:2^{\fml{F}}\to\mbb{R}$,
which is defined by,
\begin{equation} \label{eq:cfs}
  \cfn{e}(\fml{S};\fml{E}) ~~ := ~~
  \exv[\kappa\,|\,\mbf{x}_{\fml{S}}=\mbf{v}_{\fml{S}}]
\end{equation}
%
%and where $\Upsilon$ (used in the definition of the expected value) is
%defined by~\eqref{eq:upsilon}.
%
%
Thus, given a set $\fml{S}\subseteq\fml{F}$ of features,
$\cfn{e}(\fml{S};\fml{E})$ represents the expected value
%average value
of the classifier over the points of feature space represented by
$\Upsilon(\fml{S};\mbf{v})$.
%~\footnote{A more simplified notation for 
%$\frac{1}{\Pi_{i\in\fml{F}\setminus\fml{S}}|\mbb{D}_i|}\sum\nolimits_{\mbf{x}\in\Upsilon(\fml{S};\mbf{v})}\kappa(\mbf{x})$
%is $\mbf{E}[\kappa|_{\mbf{v}_{\fml{S}}}]$}.
%
The formulation presented in earlier
work~(\citealp{barcelo-aaai21}; \citealp{barcelo-jmlr23}) allows for
different input distributions when computing the average values. For
the purposes of this paper, it suffices to consider solely a uniform
input distribution, and so the dependency on the input distribution is
not accounted for.
Independently of the distribution considered, it should be clear that
in most cases $\cfn{e}(\emptyset)\not=0$; this is the case for example
with boolean classifiers~(\citealp{barcelo-aaai21};
\citealp{barcelo-jmlr23}).

To simplify notation, we use the following definitions:
\begin{align}
  \Delta_i(\fml{S}; \fml{E}) & :=
  \left(\cfn{e}(\fml{S};\fml{E})-\cfn{e}(\fml{S}\setminus\{i\};\fml{E})\right)
  \label{eq:def:delta}
  \\[2pt] %; \fml{M},\mbf{v}
  \varsigma(|\fml{S}|) & :=
  %\sfrac{|\fml{S}|!(|\fml{F}|-|\fml{S}|-1)!}{|\fml{F}|!} %;\fml{M},\mbf{v}
  \sfrac{1}{\left(|\fml{F}|\times\binom{|\fml{F}|-1}{|\fml{S}|-1}\right)}
  %|\fml{S}|!(|\fml{F}|-|\fml{S}|-1)!}{|\fml{F}|!} %;\fml{M},\mbf{v}
  \label{eq:def:vsigma}
\end{align}

Finally, let $\svn{E}:\fml{F}\to\mbb{R}$, i.e.\ the SHAP score for
feature $i$, be defined by,
%\footnote{%
%Throughout the paper, the definitions of $\Delta$ and $\sv$ are
%explicitly associated with the characteristic function used in their
%definition.}
%
\begin{align} \label{eq:sv}
  \svn{E}(i;\fml{E}) := %\\
  %&
  \sum\nolimits_{\fml{S}\in\{\fml{T}\subseteq\fml{F}\,|\,i\in\fml{T}\}}\varsigma(|\fml{S}|)\times\Delta_i(\fml{S};\fml{E})
  %\nonumber
  %%\sum\nolimits_{\fml{S}\subseteq(\fml{F}\setminus\{i\})}\varsigma(|\fml{S}|)\times\Delta_i(\fml{S};\fml{E},\cfn{e}) %;\fml{M},\mbf{v}
\end{align}
Given an instance $(\mbf{v},c)$, the SHAP score assigned to each
feature measures the \emph{contribution} of that feature with respect
to the prediction. 
From earlier work, it is understood that a positive/negative value
indicates that the feature can contribute to changing the prediction,
whereas a value of 0 indicates no
contribution~\citep{kononenko-jmlr10}.
%

%\paragraph{Formal explainability.}
%\paragraph{Logic-based explanations.}
%\paragraph{Model-based abductive and contrastive explanations.}
\paragraph{Abductive and contrastive explanations.}
%~\\
%
Given an explanation problem $\fml{E}$, a \emph{weak abductive
explanation} (WAXp) is a set of features $\fml{S}$ such that
%which, if assigned the values dictated by $\mbf{v}$,
%the probability of $\kappa(\mbf{x})=c$ is equal to 1,
the prediction remains unchanged
when the features in $\fml{S}$ are assigned the values dictated by
$\mbf{v}$:
%%then the expected value is $\kappa(\mbf{v})$, i.e.:
%
\begin{equation} \label{eq:waxp}
  %\prob(\kappa(\mbf{x})=c\,|\,\mbf{x}_{\fml{S}}=\mbf{v}_{\fml{S}})=1
  \forall(\mbf{x}\in\mbb{F}).\left(\land_{i\in\fml{S}}x_i=v_i\right)\limply\left(\kappa(\mbf{x})=\kappa(\mbf{v})\right)
\end{equation}
%
%which implies the condition
%$\exv[\kappa(\mbf{x})\,|\,\mbf{x}_{\fml{S}}=\mbf{v}_{\fml{S}}]=c$
%in the case of numerical classes.
%
(The fixed features in $\fml{S}$ are said to be \emph{sufficient} for
the given prediction.)
Moreover, an \emph{abductive explanation} (AXp) is a subset-minimal
WAXp.
Similarly to the case of AXps, a \emph{weak contrastive explanation}
(WCXp) is a set of features $\fml{S}$ such that the prediction can
change
%the probability of $\kappa(\mbf{x})=c$ is less than 1,
when the features not in $\fml{S}$ are assigned the values dictated by
$\mbf{v}$: 
%which, if allowed to change take values other than those dictated by
%$\mbf{v}$, then the probability of $\kappa(\mbf{x})=c$ is not 1:
%%expected value is not $\kappa(\mbf{v})$, i.e.:
%
\begin{equation} \label{eq:wcxp}
  %\prob(\kappa(\mbf{x})=c\,|\,\mbf{x}_{\fml{F}\setminus\fml{S}}=\mbf{v}_{\fml{F}\setminus\fml{S}})<1
  \exists(\mbf{x}\in\mbb{F}).\left(\land_{i\in\fml{F}\setminus\fml{S}}x_i=v_i\right)\land\left(\kappa(\mbf{x})\not=\kappa(\mbf{v})\right)
\end{equation}
%which is implied by the condition
%$\exv[\kappa(\mbf{x})\,|\,\mbf{x}_{\fml{F}\setminus\fml{S}}=\mbf{v}_{\fml{F}\setminus\fml{S}}]\not=c$
%in the case of numerical classes.
%
A \emph{contrastive explanation} (CXp) is a subset-minimal WCXp.
(Weak) (abductive/contrastive) explanations are associated with 
predicates $\waxp,\axp,\wcxp,\cxp:2^{\fml{F}}\to\{0,1\}$ that hold
true when the respective condition is satisfied.
Progress in formal explainability is summarized in recent
overviews~(\citealp{msi-aaai22}; \citealp{ms-rw22};
\citealp{darwiche-lics23}).

A set of features is an AXp iff it is a minimal-hitting set (MHSes) of
the set of CXps~\citep{inams-aiia20}. Similarly, a set of features is
a CXp iff it is an MHS of the set of AXps. This result is referred to as
MHS-duality, and builds on Reiter's seminal work on model-based
diagnosis~\citep{reiter-aij87}.

%Even though the paper defines AXps/CXps using probabilities (and
%expected values), these definitions are equivalent to those used in
%earlier
%works~(\citealp{inams-aiia20}; \citealp{kutyniok-jair21};
%\citealp{msi-aaai22}; \citealp{ms-rw22}).
%%; \citealp{darwiche-lics23}. 
%%% ToDo: \cite{inams-aiia20,himms-kr21,msi-aaai22}
%%
%%(The rationale for the alternative definitions~\eqref{eq:waxp}
%%and~\eqref{eq:wcxp} will become apparent in the following sections.)
%

\paragraph{Feature (ir)relevancy.}
%~\\
%
The set of features that are included in at least one (abductive) 
explanation are defined as follows:
\begin{equation}
  \mathfrak{F}(\fml{E}):=\{i\in\fml{X}\,|\,\fml{X}\in2^{\fml{F}}\land\axp(\fml{X})\}
\end{equation}
%
%where predicate $\axp(\fml{X})$ holds true if $\fml{X}$ is an AXp.
(It is known that $\mathfrak{F}(\fml{E})$ remains unchanged if CXps
are used instead of AXps~\citep{inams-aiia20}.)
%, in which case predicate $\cxp(\fml{X})$ holds true if $\fml{X}$ is a CXp.
%
%
%The sets of AXps/CXps are defined as follows:
%%
%\begin{eqnarray}
%  \mbb{A}(\fml{E}) := \{\fml{X}\in2^{\fml{F}}\,|\,\axp(\fml{X})\} \\[1.0pt]
%  \mbb{C}(\fml{E}) := \{\fml{X}\in2^{\fml{F}}\,|\,\cxp(\fml{X})\} %\\
%\end{eqnarray}
%
%where $\axp(\fml{X})$ (resp.~$\cxp(\fml{X})$) holds true if $\fml{X}$
%is an AXp (resp.~CXp).
%
%
Finally, a feature $i\in\fml{F}$ is \emph{irrelevant}, i.e.\ predicate
$\irrelevant(i)$ holds true, if $i\not\in\mathfrak{F}(\fml{E})$;
otherwise feature $i$ is \emph{relevant}, and predicate $\relevant(i)$
holds true. 
Clearly, given some explanation problem $\fml{E}$,
$\forall(i\in\fml{F}).\irrelevant(i)\leftrightarrow\neg\relevant(i)\,$.
%%
%\footnote{As noted earlier, parameterizations are elided.}

%Moreorever, we let
%$\mbb{W}_A$ denote the set of all AXps
%and
%$\mbb{W}_C$ denote the set of all WCXps.

%\paragraph{Classification problems.}
%~\\
%
%\paragraph{Logic-based explanations.}
%~\\
%
%\paragraph{SHAP scores.}
%~\\

%\paragraph{Power indices.}
\paragraph{A priori voting power.}
%~\\
A priori voting power~(\citealp{machover-psr04};
\citealp{machover-hscv15}) measures how each voter,
taken from some assembly of voters, has the potential to impact the
voting outcomes of coalitions made of voters from that assembly.
A \emph{weighted voting game} (WVG) is defined for an assembly of
voters $\fml{F}=\{1,\ldots,m\}$, where each voter $i\in\fml{F}$ has
some voting power $w_i$ (which we assume to be a non-negative
integer). Furthermore, the game is characterized by a quota
$Q\le\sum_{i\in\fml{F}}w_i$. A subset of the voters
$\fml{S}\subseteq\fml{F}$ is a coalition voting the same way,
i.e.\ either in favor or against. A coalition is a \emph{winning
coalition} if $\sum_{i\in\fml{S}}w_i\ge{Q}$; otherwise it is a losing
coalition.
A WVG $\fml{W}$ is represented by
$[Q;w_1,\ldots,w_m]$~\citep{machover-hscv15}.
Given a WVG, a \emph{power index} (or just \emph{index}),
$\pidx:\fml{F}\to\mbb{R}$, maps each voter to a real number, measuring
the voter's capacity to influence voting decisions %winning coalitions
relative to other voters.
In the more general setting of characteristic function
games~\citep{elkind-bk12}, power indices are referred to as
\emph{payoff vectors}.

%The term \emph{power indices} denotes different measures of a priori
%voting power studied over the years.
%
%As an example, Shapley values have also been studied in the context of
%a priori voting power~\citep{shapley-apsr54}.

%% file: relw.tex
\section{Related Work}

%\subsection{Power Indices}
\paragraph{Power indices.}
%~\\
%
%A priori voting power is concerned with assigning relative importance
%to each voting member of some assembly, depending on the outcome of
%each possible coalition in which the voting member can
%participate~\citep{machover-psr04}.
%
Work on a priori voting power traces back at least to the mid
1940s~\citep{penrose-jrss46}.
Over the years, different power indices have been proposed as
mechanisms for assigning relative importance to voting members of an
assembly.
While comprehensive lists of power indices
exist~(\citealp{andjiga-msh03}; \citealp{machover-psr04};
\citealp{machover-hscv15}), in this 
paper we only consider some of the best-known power indices:
Shapley-Shubik~\citep{shapley-apsr54},
Banzhaf~\citep{banzhaf-rlr65},
Deegan-Packel~\citep{deegan-ijgt78},
Johnston~\citep{johnston-ep78},
Holler-Packel~\citep{holler-je83},
Andjiga~\citep{andjiga-msh03}, and the
Responsibility index~\citep{izza-aaai24} (which can be related with
earlier work~\citep{chockler-acm-tocl08}).
For more details on power indices the reader is invited to look at the more general
setting of cooperative game theory~\citep{elkind-bk12}.

\jnoteF{Add table with succinct definitions.}

%\subsection{Axiomatic Aggregations}
\paragraph{Axiomatic aggregations.}
%~\\
%
%\jnote{Brief overview of work of Ignatiev \& co-authors, and Izza \&
%  co-authors.}
%
Motivated by a growing number of examples where SHAP scores yield
misleading relative feature importance, prior work has proposed
ways of computing measures of relative feature importance that address
the limitations of existing SHAP scores, namely \emph{formal feature
attribution}~(\citealp{ignatiev-corr23a}; \citealp{ignatiev-corr23b}),
and selected \emph{axiomatic aggregation schemes}~(\citealp{izza-corr23};
\citealp{izza-aaai24}), all of which are based on abductive
explanations. 
Even more importantly, recent work has demonstrated that formal feature
attribution corresponds to some proposed axiomatic
aggregation schemes~\citep{izza-aaai24}, and that these correspond to a few
well-known power indices~(\citealp{deegan-ijgt78};
\citealp{holler-je83}; \citealp{chockler-acm-tocl08}).

%\subsection{Corrected SHAP Scores}
\paragraph{Corrected SHAP scores.}
%~\\
%
%Furthermore, o
Other recent work observes that the problems with existing
SHAP scores result from the characteristic function
used~\citep{lhms-corr24a}, and proposes alternative characteristic
functions, which eliminate the known limitations of SHAP scores.

\jnoteF{Brief overview of our own work.}

%\subsection{Power Indicates in XAI -- Feature Aggregation Indices}
%\subsection{Decision Significance Indices (DSIs)}
%\subsection{Feature Attribution Indices (\piacronympl)}

%% file: formpi.tex
%\section{Formalizations of FAIs}
\section{Formalizations of \piacronympl}
\label{sec:formpi}

This section formalizes the best-known power indices, but targets 
the context of explainability by feature attribution. The proposed
formalizations will be referred to as \emph{\pitermscap}
(\piacronympl).

\paragraph{Why power indices for feature attribution?}
%~\\
Most of the existing work on using Shapley values in explainability
relates with the original work of L.~Shapley for assigning values to
the players of an $n$-player cooperative game. Furthermore, and as
noted earlier in~\cref{sec:prelim}, the characteristic function used
in the case of explainability has been the expected value of the ML
model~(\citealp{lundberg-nips17}; \citealp{vandenbroeck-jair22};
\citealp{barcelo-jmlr23}).%
\footnote{%
%Evidently, in the case of
For classification problems with unordered classes, % are unordered,
the expected value would be ill-defined.}
Recent work demonstrated that such characteristic function can produce
unsatisfactory relative orders of feature importance.%
\footnote{It should be unsurprising that not all characteristic
functions may be suitable for the purposes of explainability by
feature attribution using Shapley values. The question is then which
characteristic function(s) might be the most adequate.}
%
%The similarities between the goals of power indices
%and XAI by feature attribution are plain. As a result, the study of
%power indices and the characteristic functions used by power indices
%are both of interest in XAI by feature attribution. This is also
%proposed in recent work~\citep{izza-aaai24}.
%
In fact, it is apparent that feature importance is far more related
with a priori voting power than with arbitrary $n$-player games,
which impose no restrictions on the characteristic function used.
This relationship is exploited for example in recent
work~\citep{izza-aaai24}.
Moreover, early work on voting power was clear about how to measure the
importance of each voter in voting games. Indeed, and quoting
from~\citep{shapley-apsr54}:
``\emph{Our definition of the power of an individual member depends on
the chance he has of being critical to the success of a winning
coalition}''. All of the power indices briefly summarized above
explicitly measure the importance %of how likely %account for whether
of an individual member (in our case a feature) at being
\emph{critical}.
The same insight can be used in explainability by feature attribution.
We can view feature attribution as a voting game, as follows. A set of
fixed features represents a coalition. A set of fixed features is a
winning coalition if those features are sufficient for the given
prediction. Furthermore, and adapting the quote 
above~\citep{shapley-apsr54}: \emph{the importance of a feature  
depends on the chance it has of being critical for a set of fixed
features to be sufficient for a given prediction.}
Moreover, it is plain to conclude that power indices and their use in
explainability will yield fundamentally different measures of feature
importance than what earlier work on SHAP scores
did~\citep{lundberg-nips17}.
Finally, one could argue that capturing feature attribution by
mimicking power indices should not be warranted. However, as the most
widely used version of SHAP scores confirms, by choosing a problematic 
characteristic function, one may obtain unsatisfactory results.

%\subsection{Why FISs?}%
\paragraph{Why FISs?}%
In cooperative game theory, each characteristic function defines a
concrete game. In this case, the Shapley values assign relative
importance to each player, for the given game~\citep{shapley-ctg53}. 
In the case of a priori voting power, the (usually implicit)
characteristic function captures the notion of a voter being critical
for a winning coalition; this is often reflected in terms of the
winning coalitions that are accounted for. 
In the case of feature attribution for XAI, the situation is somewhat
more complex. We adopt the approach used for power indices, namely we
will look at the cases when a feature is critical for a winning
coalition (i.e.\ fixed features sufficient for prediction), and such
that such coalition may or may not be subset-minimal.
In addition, we will adapt the best-known power indices. However, we
associate a parameterizable characteristic function with each FIS,
which may or may not be left implicit. Then, each explanation problem
is used to set the characteristic function to consider, and as a
result produce the FIS associated with the different features.

In contrast with other work~\citep{izza-aaai24}, \piacronympl will be
parameterized given an explanation problem and a chosen characteristic
function. Afterwards, we detail the assumptions that must be made to
obtain the indices proposed elsewhere~(\citealp{ignatiev-corr23a};
\citealp{izza-corr23}; \citealp{ignatiev-corr23b};
\citealp{izza-aaai24}).

%\subsection{Core Formalizations}
\subsection{Formalizing Template Scores}

We will use the following concepts. First,
$\cf:2^{\fml{F}}\to\mbb{R}$ is a characteristic function. Furthermore,
we parameterize $\cf$ using the target explanation problem $\fml{E}$,
i.e.\ when defining the characteristic function, we will use the
notation $\cf(\fml{S};\fml{E})$, indicating that we are computing the
characteristic function for some set $\fml{S}$, and that the
characteristic function is parameterized on $\fml{E}$. %
%
%\footnote{%
%It is important to underscore the crucial differences between
%\piacronympl and cooperative games and also power indices.
%For cooperative games, the characteristic function represents the game
%being studied; this is the case for example with Shapley
%values~\citep{shapley-ctg53}. In this case, the characteristic
%function is a given of the problem being studied.
%%
%For power indices, the characteristic function represents some notion
%of criticality (of a voter) and/or minimality (of a set of voters).
%Once again, the characteristic function is a given; this is further
%analyzed later in this section.
%%
%In contrast, for \piacronympl, the characteristic function \emph{must}
%account for the target explanation problem (and so also for the ML
%model being explained). As a result, dependency of the characteristic
%function on the explanation problem must be made explicit.}
%
Given $\fml{S}\subseteq\fml{F}$ and $i\in\fml{S}$, we define,%
\footnote{%
With respect to~\eqref{eq:def:delta}, $\Delta_i$ is being redefined to
allow for a parameterization on the characteristic function $\cf$.}
\begin{equation} \label{eq:def:delta2}
  \Delta_i(\fml{S};\fml{E},\cf)=\cf(\fml{S};\fml{E})-\cf(\fml{S}\setminus\{i\};\fml{E})
\end{equation}
$\Delta_i(\fml{S};\fml{E},\cf)$ denotes the influence of $\fml{S}$ due to
feature $i$.
Moreover,
$\Delta(\fml{S};\fml{E},\cf)=\sum_{i\in\fml{S}}\left(\Delta_i(\fml{S};\fml{E},\cf)\right)$
denotes the \emph{relative influence} of $\fml{S}$.
(Although the definition of $\Delta_i$ may seem different from recent
works on
explainability~(\citealp{lundberg-nips17};
\citealp{vandenbroeck-jair22}; \citealp{barcelo-jmlr23})
it follows the original work in game theory~\citep{shapley-ctg53} and a
priori voting power~\citep{shapley-apsr54}. More importantly, it is
plain to conclude that either can be used.)

Furthermore, the following additional sets are defined:
\begin{align}
  \mbb{WA}_i(\fml{E})&=\{\fml{S}\subseteq\fml{F}\,|\,\waxp(\fml{S};\fml{E})\land{i\in\fml{S}}\}\nonumber\\
  \mbb{WC}_i(\fml{E})&=\{\fml{S}\subseteq\fml{F}\,|\,\wcxp(\fml{S};\fml{E})\land{i\in\fml{S}}\}\nonumber\\
  \mbb{A}_i(\fml{E})&=\{\fml{S}\subseteq\fml{F}\,|\,\axp(\fml{S};\fml{E})\land{i\in\fml{S}}\}\nonumber\\
  \mbb{C}_i(\fml{E})&=\{\fml{S}\subseteq\fml{F}\,|\,\cxp(\fml{S};\fml{E})\land{i\in\fml{S}}\}\nonumber%\\
\end{align}
The definitions of $\mbb{WA}$, $\mbb{WC}$, $\mbb{A}$, and $\mbb{C}$
mimic the ones above, but without specifying any feature.

%%$S\in\{\fml{T}\,|\,\fml{T}\subseteq\fml{F}\land{i}\in\fml{T}\}$

\paragraph{Defining template scores.} %\piacronympl
%~\\
%
%Feature attribution indices (FAI) represent the use of power indices in XAI.
\piacronympl will be represented with the notation $\sv$, standing for
\emph{score}.
An \piacronym is a function mapping features to some real
value, $\sv:\fml{F}\to\mbb{R}$.
Moreover, each \piacronympl $\sv$ will be the result of instantiating
some \emph{template score} $\tsv$, by specifying a concrete
characteristic function. Furthermore, each template score will
generalize a well-known power index.%
\footnote{%
In cases where the characteristic function is not parameterized on
some explanation problem, then a template score denotes an index. The
term template score serves to highlight that such indices can be
parameterized by some explanation problem.
Moreover, as shown in~\cref{sec:nidx}, our framework enables the
definition of FISs without starting from a template score. However, we
first show how we obtain FISs from template scores.}
%
%In addition, each \piacronym is obtained
%from a template score by specifying a concrete characteristic function
%to use.
%
Moreover, and besides resulting from a specific instantiation of a
template score, each \piacronym is parameterized on the target
explanation problem $\fml{E}$.
%
%Several generalized definitions of \piacronympl are 
%presented next.%
%Furthermore, \piacronympl are parameterized on both the target
%explanation problem $\fml{E}$ and a chosen characteristic function
%$\cf$ (which may be left unspecified).
%
Several generalized template scores (from which we will later derive
\piacronympl) are presented next.% 
\footnote{%
This paper introduces generalized formalizations of the best-known
power indices, starting from existing definitions
% rigorous definitions %formalizations
of power indices~(e.g.~\citep{andjiga-msh03}). Later, the paper shows
how suitable characteristic functions can be used to exactly capture
the original power indices. %\\ %, in different settings
%
%%As before, parameterizations will be elided if possible.
}
%
%\footnote{%
%As before, parameterizations will be elided if possible.}

\paragraph{Template score using Shapley-Shubik's index.}
%~\\
Following~\citep{shapley-apsr54}, the template Shapley-Shubik score is
defined as follows:
\begin{equation} \label{eq:tsss}
\tsvn{S}(i; \fml{E}, \cf) :=
\sum_{\fml{S}\in\{\fml{T}\subseteq\fml{F}\,|\,i\in\fml{T}\}}
\left(\frac{\Delta_i(\fml{S};\fml{E},\cf)}{|\fml{F}|\times\binom{|\fml{F}|-1}{|\fml{S}|-1}}\right)
%{\fml{S}\in\{\fml{T}\subseteq\fml{F}\,|\,i\in\fml{T}\}}
\end{equation}

\paragraph{Template score using Banzhaf's index.}
%~\\
Following~\citep{banzhaf-rlr65}, the template Banzhaf score is defined
as follows:
\begin{equation} \label{eq:tsb}
\tsvn{B}(i;\fml{E},\cf) := \sum_{\fml{S}\in\{\fml{T}\subseteq\fml{F}\,|\,i\in\fml{T}\}}
\left(\frac{\Delta_i(\fml{S};\fml{E},\cf)}{2^{|\fml{F}|-1}}\right)
\end{equation}

\paragraph{Template score using Johnston's index.}
%~\\
Following~\citep{johnston-ep78}, the template Johnston score is
defined as follows:
\begin{equation} \label{eq:tsj}
  \tsvn{J}(i;\fml{E},\cf) := \sum_{\substack{\fml{S}\in\{\fml{T}\subseteq\fml{F}\,|\,i\in\fml{T}\}\\\Delta(\fml{S};\fml{E},\cf)\not=0}}\left(\frac{\Delta_i(\fml{S};\fml{E},\cf)}{\Delta(\fml{S};\fml{E},\cf)}\right)
\end{equation}

\paragraph{Template score using Deeghan-Packel index.}
%~\\
Following~\citep{deegan-ijgt78}, the template Deegan-Packel score is
defined as follows:
\begin{equation} \label{eq:tsdp}
\tsvn{D}(i;\fml{E},\cf) := \sum_{\fml{S}\in\mbb{A}_i(\fml{E})}
\left(\frac{\Delta_i(\fml{S};\fml{E},\cf)}{|\fml{S}|\times|\mbb{A}(\fml{E})|}\right)
\end{equation}

\paragraph{Template score using Holler-Packel index.}
%~\\
Following~\citep{holler-je83}, the template Holler-Packel score is
defined as follows:
\begin{equation} \label{eq:tshp}
\tsvn{H}(i;\fml{E},\cf) := \sum_{\fml{S}\in\mbb{A}_i(\fml{E})}
\left(\frac{\Delta_i(\fml{S};\fml{E},\cf)}{|\mbb{A}(\fml{E})|}\right)
\end{equation}

\paragraph{Template score using the Responsibility index.}
%~\\
Following~\citep{izza-aaai24}, the template Responsibility
score is defined as follows:
\begin{equation} \label{eq:tsr}
\tsvn{R}(i;\fml{E},\cf) :=
\max\left\{\left.\frac{\Delta_i(\fml{S};\fml{E},\cf)}{|\fml{S}|}\,\right|\,\fml{S}\in\mbb{A}_i(\fml{E})\right\}
\end{equation}

The Deegan-Packel, Holler-Packel and Responsibility scores have been
recently applied in the context of explainability~\citep{izza-aaai24}.
However, as shown below, our definitions generalize the ones
in~\citep{izza-aaai24}. Furthermore, the relationship between the
Responsibility score and the Chockler
index~\citep{chockler-acm-tocl08} will be analyzed later
in~\cref{sec:nidx}.

\paragraph{Template score using Andjiga's index.}
%~\\
Following~\citep{andjiga-msh03}, the template Andjiga score is defined as
follows:
\begin{equation} \label{eq:tsa}
\tsvn{A}(i;\fml{E},\cf) := \sum_{\fml{S}\in\mbb{WA}_i}
\left(\frac{\Delta_i(\fml{S};\fml{E},\cf)}{|\fml{S}|\times|\mbb{WA}(\fml{E})|}\right)
\end{equation}

\jnoteF{In each case, summarize existing text-based definitions.}

\paragraph{Other power indices.}
%~\\
%
The importance of power indices for a priori voting power (and also
characteristic function games) has motivated the proposal of
additional proposals over the years~(\citealp{penrose-jrss46};
\citealp{coleman-sc71}; \citealp{curiel-mss87};
\citealp{colomer-ejpr96}; \citealp{berg-gdn99};
\citealp{freixas-cgtm10}; \citealp{freixas-ejor14};
\citealp{felsenthal-ho16}). Our assessment is that these do not offer
additional insights for explainability by feature attribution.
Nevertheless, it would be simple to devise generalized template scores
using those additional power indices.

\paragraph{From template to instantiated scores.}
%~\\
%
Given one of the template scores $\tsv$ detailed above, we can
instantiate a concrete score by specifying the characteristic function
to use. (Clearly, we could consider arbitrary many characteristic
functions.) Finally, we obtain an FIS by parameterizing the concrete
score given a target explanation problem.

Suppose we have a template score $\tsvn{T}$ parameterized on some
explanation problem $\fml{E}$ and (unspecified) characteristic
function $\cfn{\tau}$. 
%
%Suppose we pick some concrete characteristic function $\cf{E}$.
Then, the instantiated score $\svn{I}$ will be defined as follows: 
\[
\svn{I}(i;\fml{E}) := \tsvn{T}(i; \fml{E},\cfn{\tau})
\]

\subsection{Existing Scores as \piacronympl}
%\paragraph{Existing scores.}

The most often studied feature importance score, also referred to as
the (original) SHAP score is given by~\eqref{eq:sv}, and represented
by $\svn{E}$.
Clearly, $\svn{E}$ can be viewed as an instantiation of the
template score based on Shapley-Shubik index, $\tsvn{S}$, using
$\cfn{e}$ (see~\eqref{eq:cfs}):
%
%\[
$\svn{E}(i;\fml{E}) := \tsvn{S}(i; \fml{E}, \cfn{e})$.
%\]
%
% Accordingly, we will refer to this score as $\svn{E}$.
%

Recent work proposed a workaround to address some of the limitations
of existing SHAP scores, by replacing the classification function with an
alternative similarity (or matching) %indistinguishability
function~\citep{lhms-corr24a}.
\[
\zeta(\mbf{x};\fml{E}) = \left\{
\begin{array}{lcl}
  1 & \quad & \tn{if $(\kappa(\mbf{x})=\kappa(\mbf{v}))$} \\[2pt]
  0 & \quad & \tn{otherwise} \\
\end{array}
\right.
\]

Given the similarity function, we introduce a new characteristic
function,
\begin{equation}
  %\cfn{m}(\fml{S};\fml{E}) ~~ := ~~ %%\quad \triangleq  \quad
  %\exv[\zeta(\mbf{x})\,|\,\mbf{x}\in\Upsilon(\fml{S};\mbf{v})]
  %%= \varphi(\fml{S};\fml{E})
  \cfn{m}(\fml{S};\fml{E}) ~~ := ~~
  \exv[\zeta(\mbf{x})\,|\,\mbf{x}_{\fml{S}}=\mbf{v}_{\fml{S}}]
  \label{eq:def:vs}
\end{equation}

%Accordingly, $\Delta_i(\fml{S};\fml{E},\cfn{m})$ is given by,
%\[
%\Delta_{i}(\fml{S};\fml{E},\cfn{m}):=
%\cfn{m}(\fml{S}\cup{i};\fml{E})
%-
%\cfn{m}(\fml{S};\fml{E})
%\]

%\begin{align} \label{eq:tsv}
%  \tsvn{E}(i;\fml{E},\cf&):=\\
%  &\sum\nolimits_{\fml{S}\subseteq(\fml{F}\setminus\{i\})}\varsigma(|\fml{S}|)\times\Delta_i(\fml{S};\fml{E},\cfn{m}) %;\fml{M},\mbf{v}
%  \nonumber
%\end{align}

And, finally, the FIS using the similarity (or matching) function
becomes:
\[
\svn{M}(i;\fml{E}) := \tsvn{S}(i;\fml{E},\cfn{m})
\]

%Finally, we change the definition of $\svn{E}$ by replacing $\kappa$
%with $\zeta$, to obtain $\svn{M}$.
%%
%\begin{align} \label{eq:sv}
%  \svn{M}(i;\fml{E},\cfn{e}&):=\\
%  &\sum\nolimits_{\fml{S}\subseteq(\fml{F}\setminus\{i\})}\varsigma(|\fml{S}|)\times\Delta_{M}(i,\fml{S};\fml{E},\cfn{m}) %;\fml{M},\mbf{v}
%  \nonumber
%\end{align}
%%

%\subsection{Capturing Axiomatic Aggregations} \label{ssec:axioaggr}
\subsection{Axiomatic Aggregations as \piacronympl} \label{ssec:axioaggr}

Recent work~\citep{izza-aaai24} showed that two variants of formal
feature attribution~(\citealp{ignatiev-corr23a};
\citealp{ignatiev-corr23b}) can be viewed as instantiations of the
Deegan-Packel and Holler-Packel scores. In addition, a third power
index was also studied, namely the so-called Responsibility score.
Although this earlier work might seem to involve simpler
formalizations than the ones used in~\eqref{eq:tsdp},
~\eqref{eq:tshp}, ~\eqref{eq:tsr},
%in the previous section
because no
characteristic function was explicitly considered, we now show that it
suffices to pick the right characteristic function to obtain the same
scores.%
\footnote{%It should be noted that this earlier
Earlier work~\citep{izza-aaai24} refers to indices and not to scores.
%As justified earlier in the paper,
We will opt to use the term scores, and more specifically FISs, in the
case of feature attribution in XAI.}

\paragraph{Choosing a characteristic function.}
%~\\
%
The recently proposed axiomatic aggregations based on AXps are special
cases of the power indices introduced above, in the case of the
Deegan-Packel, Holler-Packel and Responsibility indices, when a
concrete characteristic function is used. We show below that this is
indeed the case.

Let us define the following characteristic function:
$\cfn{a}(\fml{S};\fml{E}) := \tn{ITE}(\axp(\fml{S};\fml{E}),1,0)$.
%
%\[
%\cfn{a}(\fml{S};\fml{E}) :=
%\left\{
%\begin{array}{lcl}
%  1 & \quad & \tn{if $\axp(\fml{S};\fml{E})$} \\[1.5pt]
%  0 & \quad & \tn{otherwise} \\
%\end{array}
%\right.
%\]
%
%Furthermore, we let $\cfn{H}=\cfn{a}$ and $\cfn{R}=\cfn{a}$.
As a result, the Deegan-Packel, Holler-Packel and Responsibility
indices can be reformulated as FISs, as follows:
\begin{align}
  \svn{D}(i;\fml{E}) & :=
  \tsvn{D}(i;\fml{E},\cfn{a})
  =\sum\nolimits_{\fml{S}\in\mbb{A}_i(\fml{E})}
  \left(\sfrac{1}{(|\fml{S}|\times|\mbb{A}(\fml{E})|)}\right)
  \nonumber\\[1.5pt]
  \svn{H}(i;\fml{E}) & :=
  \tsvn{H}(i;\fml{E},\cfn{a}) = 
  \sum\nolimits_{\fml{S}\in\mbb{A}_i(\fml{E})}
  \left(\sfrac{1}{|\mbb{A}(\fml{E})|}\right)
  \nonumber\\[1.5pt]
  \svn{R}(i;\fml{E}) & :=
  \tsvn{R}(i;\fml{E},\cfn{a}) =
  \max\left\{\left.\sfrac{1}{|\fml{S}|}\,\right|\,\fml{S}\in\mbb{A}_i(\fml{E})\right\}
  \nonumber%\\
\end{align}
These modified definitions match exactly those used in recent
work~\citep{izza-aaai24}.

\paragraph{Examples of unsatisfactory attributions.}
%~\\
%
For some of the scores studied in recent work~\citep{izza-aaai24},
there are simple examples that illustrate that the computed scores can
be unsatisfactory in terms of relative order of feature importance. We
analyze two examples. 

\begin{example}[Holler-Packel FIS, $\svn{H}$]
  Let the set of AXps be
  $\mbb{A}=\{\{1\},\{2,3,4,5,6\},\{2,3,4,5,7\}\}$.
  It is the case that
  $\svn{H}(1)=\svn{H}(6)=\svn{H}(7)<\svn{H}(2)=\svn{H}(3)=\svn{H}(4)=\svn{H}(5)$,
  but it is apparent that feature 1 should be deemed more important
  than any of the other features.
\end{example}

\begin{example}[Responsibility FIS, $\svn{R}$]
  Let the set of AXps be,
  \[
  \begin{array}{lcl}
    \mbb{A} & = & \{\{1,8\},\{2,3,4\},\{2,3,5\},\{2,3,6\},\\
    & & ~~\{2,3,7\},\{2,4,5\},\{2,4,6\},\{2,4,7\},\\
    & & ~~\{2,5,6\},\{2,5,7\},\{2,5,6\}\}
  \end{array}
  \]
  It is the case that $\svn{R}(1)=\svn{R}(8)>\svn{R}(2)$. However, it 
  is apparent that feature 2 should be deemed more important than
  feature 1; feature 2 occurs in ten size 3 AXps, and features 1 and 8
  only occur in one size 2 AXp.
\end{example}

For the Deegan-Packel FIS, it is more difficult to find simple
explanation problems for which clearly unsatisfactory results are
obtained.
%As a result, to uncover an unsatisfactory relative order of feature
%importance, a
%(For completeness, a more sophisticated case study, which reveals an
%unsatisfactory relative order of feature importance, is included in
%the Supplementary Materials.) %~\cref{app:uadp}.
%
Nevertheless, the Deegan-Packel \piacronym does not respect a few
relevant properties of \piacronympl, as shown in~\cref{sec:tres}.

\begin{comment}
%
%
\paragraph{Additional limitations.}
%~\\
%
Besides the limitations of the indices themselves, some claims in
recent work~\citep{izza-corr23,izza-aaai24} can be disproved.
(Counterexamples to those claims are included in the Supplementary 
Materials.) %~\cref{app:disp}. 
%
%
\end{comment}

%\subsection{Representing Other Power Indices}
%\subsection{Other Power Indices as \piacronympl}
\subsection{Well-Known Power Indices as \piacronympl}

We now show how FISs can be obtained from the Shapley-Shubik, Banzhaf,
Johnston, and Andjiga indices.

%\paragraph{Discussion.}
%~\\
%

\paragraph{FISs based on abductive explanations.}
%~\\
%
The simple characteristic function proposed in~\cref{ssec:axioaggr}
serves for capturing both formal feature
attribution~(\citealp{ignatiev-corr23a}; \citealp{ignatiev-corr23b})
and axiomatic aggregations~(\citealp{izza-corr23};
\citealp{izza-aaai24}), all of which are based on AXps. However, the
same characteristic function \emph{cannot} be used to capture either
the Shapley-Shubik or the Banzhaf indices (and this may indirectly
justify why these power indices have not been related with AXps
before). %%whether a set of features is an AXp.)

As a result, we introduce a somewhat different characteristic
function:
$\cfn{w}(\fml{S};\fml{E}) := \tn{ITE}(\waxp(\fml{S};\fml{E}),1,0)$.
%
%\[
%\cfn{w}(\fml{S};\fml{E}) :=
%\left\{
%\begin{array}{lcl}
%  1 & \quad & \tn{if $\waxp(\fml{S};\fml{E})$} \\[1.5pt]
%  0 & \quad & \tn{otherwise} \\
%\end{array}
%\right.
%\]
(As shown next, $\cfn{w}$ enables capturing both the Shapley-Shubik
and the Banzhaf FISs, among others.) 

Given $\cfn{w}$, we will obtain $\Delta_i(\fml{S})=1$ only for sets
$\fml{S}\subseteq\fml{F}$ such that, 
$\waxp(\fml{S})\land\neg\waxp(\fml{S}\setminus\{i\})$. (Evidently, we
can restrict further those sets with additional constraints.)
Observe that this definition of the characteristic function captures
the cases where feature $i$ is \emph{critical} for making $\fml{S}$ a
\emph{winning coalition}, and so naturally relates with the intended
goals of Shapley-Shubik's index~\citep{shapley-apsr54}.
However, the modified definition does not equate with $\fml{S}$ being
an AXp. Indeed, feature $i$ may be critical, but other features may
not be critical, and so $\fml{S}$ might not be an AXp.

Given a set of features $\fml{S}\subseteq\fml{F}$ and a feature
$i\in\fml{F}$, with $i\in\fml{S}$, define the predicate denoting
whether a feature $i$ is critical for $\fml{S}$,
\begin{equation} \label{eq:def:critic}
  \critic(i,\fml{S};\fml{E}) :=
  \waxp(\fml{S};\fml{E})\land\neg\waxp(\fml{S}\setminus\{i\};\fml{E})
\end{equation}
As a result, we can instantiate the Shapley-Shubik, Banzhaf, Johnston
and Andjiga template scores as follows:
%{\fml{S}\subseteq\fml{F}\land\critic(i,\fml{S})}
%{\fml{S}\in\{\fml{T}\subseteq\fml{F}\,|\,\critic(i,\fml{T})\}}
\begin{align}
  \svn{S}(i;\fml{E}) & := \tsvn{S}(i;\fml{E},\cfn{w}) = \nonumber \\
  &\sum\nolimits_{\fml{S}\subseteq\fml{F}\land\critic(i,\fml{S};\fml{E})}
  \left(\sfrac{1}{\left(|\fml{F}|\times\binom{|\fml{F}|-1}{|\fml{S}|-1}\right)}\right)
  \nonumber\\[1.5pt]
  \svn{B}(i;\fml{E}) & := \tsvn{B}(i;\fml{E},\cfn{w}) = \nonumber \\
  & \sum\nolimits_{\fml{S}\subseteq\fml{F}\land\critic(i,\fml{S};\fml{E})}
  \left(\sfrac{1}{2^{|\fml{F}|-1}}\right)
  \nonumber\\[1.5pt]
  \svn{J}(i;\fml{E}) & := \tsvn{J}(i;\fml{E},\cfn{w}) = \nonumber \\
  & \sum\nolimits_{\fml{S}\subseteq\fml{F}\land\critic(i,\fml{S};\fml{E})}
  \left(\sfrac{1}{\Delta(\fml{S})}\right)
  \nonumber\\[1.5pt]
  \svn{A}(i;\fml{E}) & := \tsvn{A}(i;\fml{E},\cfn{w}) = \nonumber \\
  & \sum\nolimits_{\fml{S}\subseteq\fml{F}\land\critic(i,\fml{S};\fml{E})}
  \left(\sfrac{1}{\left(|\fml{S}|\times|\mbb{WA}|\right)}\right)
  \nonumber%\\
\end{align}
As can be concluded, by selecting a suitable characteristic function,
we are able to define \piterms that mimic some of the best-known
indices in a priori voting power.% 
\footnote{It should be noted that we could also use $\cfn{w}$ to
obtain $\svn{D}$ and $\svn{H}$, by also requiring $\fml{S}$ to be
subset-minimal. However, we opted to mimic the definitions proposed in
earlier work~\citep{izza-aaai24}, where the concept of characteristic
function is not used.}
Furthermore, we note that $\svn{E}$, $\svn{M}$ and $\svn{S}$ are
defined using the same template score, which generalizes the
Shapley-Shubik index. These scores are distinguished by the
characteristic function used for instantiating each one, and so can be
referred to as the (expected-value) SHAP score (with
$\svn{E}(i;\fml{E})=\tsvn{S}(i;\fml{E},\cfn{e})$),
the similarity (or matching) SHAP score (with
$\svn{E}(i;\fml{E})=\tsvn{S}(i;\fml{E},\cfn{m})$),
and the formal SHAP score (with
$\svn{S}(i;\fml{E})=\tsvn{S}(i;\fml{E},\cfn{w})$).
Finally, some of the characteristic functions studied earlier in this
paper, e.g.\ $\cfn{e}$ or $\cfn{m}$, could also be used for defining
alternative scores starting from the template scores proposed at the
beginning of this section, e.g.\ $\svn{D}$ or $\svn{A}$, among others.

\jnoteF{%
  Cover: $\svn{E}$ (Expected value/Lundberg's), $\svn{M}$
  (Similarity/indistinguishability), $\svn{S}$ (Formal Shapley, or
  just Shapley), etc.
}

%% file: nprop.tex
%\section{New Properties for Feature Attribution Indices}
\section{Proposed Properties for \piacronympl}
\label{sec:nprop}

As observed earlier, \piacronympl raise a number of challenges not
faced either by cooperative games or by power indices. More
concretely, in the case of \piacronympl, the selected characteristic
function will depend on the target explanation problem and so also on
the ML model being explained.
As a result, the goal of faithfully capturing feature importance
requires identifying additional properties, in addition to those that
Shapley values in cooperative games are well-known for.%
\footnote{%
Different \piacronympl may yield different rankings of features,
%The \emph{ideal} ranking of features is elusive and most likely
%subjective
similarly to what is observed with power
indices~(\citealp{freixas-ijgt10}; \citealp{moretti-ijcai22}).}
This explains in part why recent work~\citep{izza-aaai24} also
proposed additional properties in addition to those that characterize
Shapley values in cooperative game theory.

\subsection{Properties Targeting Cooperative Games}

The best-known properties for values assigned to players of
cooperative games are those proposed by L.~Shapley
\citep{shapley-ctg53}. We start by overviewing well-known properties
for indices (and so for template scores), which are parameterized on
the explanation problem and (every) characteristic function.
(We adopt the presentation from~\citep{elkind-bk12}.)

\paragraph{P01: Efficiency.} \label{par:p01}
%~\\
%
A cooperative game is efficient if,
\[
\sum\nolimits_{i\in\fml{F}}\sv(i;\fml{E},\cf)=\cf(\fml{F})-\cf(\emptyset)
\]
In game theory, it is often assumed that $\cf(\emptyset)=0$
(e.g.~\citep{shapley-ctg53}). However, in the case of SHAP scores,
$\cf(\emptyset)\not=0$. As a result, claims of efficiency for SHAP
scores usually ignore the term $\cf(\emptyset)$.

\paragraph{P02: Symmetry.}
%~\\
%
Two elements $i,j\in\fml{F}$ are symmetric if
$\cf(\fml{S}\cup\{i\})=\cf(\fml{S}\cup\{j\})$ for
$\fml{S}\subseteq\fml{F}\setminus\{i,j\}$.
An index respects the property of symmetry if, for any $\fml{E}$ and
for any symmetric elements $i,j\in\fml{F}$,
$\sv(i;\fml{E},\cf)=\sv(j;\fml{E},\cf)$.

\paragraph{P03: Additivity.}
%~\\
%
Let $\cfn{1}$ and $\cfn{2}$ represent two characteristic functions.
Furthermore, for any $\fml{S}\subseteq\fml{F}$, let
$\cfn{1{+}2}(\fml{S})=\cfn{1}(\fml{S})+\cfn{2}(\fml{S})$.
An index respects the property of additivity if for $i\in\fml{F}$, 
$\svn{1{+}2}(i;\fml{E},\cfn{1{+}2})=\svn{1}(i;\fml{E},\cfn{1})+\svn{2}(i;\fml{E},\cfn{2})$.

\paragraph{P04: Dummy player.}
%~\\
%
For any element $i\in\fml{F}$, if $\cf(\fml{S})=\cf(\fml{S}\cup\{i\})$
for any $\fml{S}\subseteq\fml{F}$, then $i$ is a \emph{dummy player}
(or element).
An index respects the dummy player property if for any dummy element
$i\in\fml{F}$, $\sv(i)=0$.

%\subsection{Properties Targeting Power Indices}
%
%As noted earlier, power indices can be viewed as generalizing values
%of importance for players in cooperative games, in that the actual
%game is defined given the target weighted voting game.

\subsection{Properties Devised for Axiomatic Aggregations}

Work on axiomatic aggregations~\citep{izza-aaai24} proposed a number
of properties that scores should respect. The properties described
below are of interest to the purposes of our work, and apply to
\piacronympl.

%\paragraph{P05: Null feature.}
%%~\\
%%
%Recent results 
%in~(\citealp{hms-corr23a};\citealp{hms-corr23d}; \citealp{hms-ijar24};
%\citealp{msh-cacm24}) showed that feature irrelevancy may not be
%respected by SHAP scores, i.e.\ irrelevant features may be assigned a
%non-zero SHAP score. Motivated by the limitation of existing SHAP
%scores, the recent work on axiomatic aggregations~\citep{izza-aaai24}
%proposed the \emph{null property}, i.e.\ an irrelevant feature should
%be assigned a score of 0.
%\[
%\forall(i\in\fml{F}).\neg\relevant(i)\limply(\sv(i)=0)
%\]

\paragraph{P05: Minimal monotonicity.}
%~\\
%
Another property proposed for axiomatic
aggregations~\citep{izza-aaai24} requires that, given any classifier,
whenever the set of AXps for feature $i$ is included in the set of
AXps for feature $j$, then the score for $i$ should be no greater than
the score for $j$,
\[
%\forall(\fml{E})
\forall(i,j\in\fml{F}).(\mbb{A}_i\subseteq\mbb{A}_j)\limply\sv(i;\fml{E},\cf)\le\sv(j;\fml{E},\cf)
\]

\paragraph{P06: $\gamma$-efficiency.}
%~\\
%
A final property used for axiomatic aggregations generalizes the
notion of efficiency.
Let $\gamma(;\fml{E})\in\mbb{R}$. A score is $\gamma$-efficient if,
\[
\sum\nolimits_{i\in\fml{F}}(\sv(i; \fml{E}, \kappa))=\gamma(;\fml{E})
\]
$\gamma$-efficiency can be viewed as a mechanism of relaxing the
property of efficiency (see P01 above), %on~\cpageref{par:p01}),
and has been the %\cref{par:p01} 
subject of past research~\citep{dubey-mor81}.

\paragraph{Other properties.}
We will not adopt other properties discussed in~\citep{izza-aaai24}.
Some of the remaining properties, e.g.\ \emph{null feature}, are
special cases of the properties discussed in the remainder of this
section.

%\paragraph{P4: Additivity.}
%~\\

\subsection{Novel Properties Targeting Scores}

As noted earlier in the paper, FISs can be viewed as a generalization
of power indices. Whereas the ML model can be viewed as representing
the weighted voting game, for each ML model there are many different
instances that need to be explained.

\paragraph{P07: Independence from class labeling.}
%~\\
%
Given a classifier $\fml{M}=(\fml{F},\mbb{F},\fml{K},\kappa)$, a class
relabeling consists in defining a new classifier
$\fml{M}=(\fml{F},\mbb{F},\fml{K}',\kappa')$, where the elements of
$\fml{K}'$ are related with the elements of $\fml{K}$ by a bijection
$\sigma:\fml{K}\to\fml{K}'$, such that
$\forall(\mbf{x}\in\mbb{F}).(\kappa'(\mbf{x})=\sigma(\kappa(\mbf{x})))$.
%
%The resulting classifier is
%$\fml{M}'=(\fml{F},\mbb{F},\fml{K}',\kappa')$.
%
For $\fml{M}$ and $\fml{M}'$, the respective explanation problems are
$\fml{E}$ and $\fml{E}'$.

Let $\sv$ and $\sv'$ represent two FISs defined, respectively, on
$\fml{E}$ and $\fml{E}'$.
Then, $\sv$ respects independence from class labeling if
$\forall(i\in\fml{F}).(\sv(i;\fml{E})=\sv(i;\fml{E}'))$.

\paragraph{P08: Consistency with relevancy.}
%~\\
%
An FIS is consistent with relevancy if, for any explanation problem
$\fml{E}$, %\forall(\fml{E}).
$\forall(i\in\fml{F}).(\relevant(i)\lequiv\sv(i;\fml{E})\not=0$,
i.e.\ a feature takes a non-zero FIS iff it is relevant.
One can readily conclude that property P08 generalizes the \emph{null
feature} property proposed in~\citep{izza-aaai24}.

We underline that properties P07 and P08 should both be deemed 
\emph{mandatory}. Otherwise, as shown in recent
work~(\citealp{hms-ijar24}; \citealp{msh-cacm24}), one can construct
examples where relative feature importance is misleading.

\paragraph{P09: Consistency with duality.}
%~\\
%
MHS duality plays a key role in the enumeration of AXps and CXps.
Furthermore, it is known that CXps can be easier to enumerate than
AXps~\citep{hiims-kr21}. As a result, it is of interest to know
whether using (W)CXps instead of (W)AXps changes the computed FISs. 
A dual characteristic function is obtained by replacing the references to
(W)AXps by references to (W)CXps. Similarly, a dual template score is
obtained by replacing the references to (W)AXps by references to
(W)CXps.
Moreover, a dual FIS is obtained by replacing the template score with
its dual template score, and using a dual characteristic function.

Given the above-defined characteristic function $\cfn{w}$, the
resulting dual characteristic function $\cfd{w}$ is:
%For example, we can define dual FISs by replacing $\cfn{w}$ with the
%following alternative characteristic function:
$\cfd{w}(\fml{S};\fml{E}) := \tn{ITE}(\wcxp(\fml{S};\fml{E}),1,0)$.
%
%\[
%\cfd{w}(\fml{S};\fml{E}) :=
%\left\{
%\begin{array}{lcl}
%  1 & \quad & \tn{if $\wcxp(\fml{S};\fml{E})$} \\[1.5pt]
%  0 & \quad & \tn{otherwise} \\
%\end{array}
%\right.
%\]
%
Furthermore, for defining the dual template scores, sets defined using
AXps will be replaced by sets defined using CXps. As an example, the
dual predicate for a critical feature becomes,
\begin{equation} \label{eq:def:criticd}
  \critic_d(i,\fml{S};\fml{E}) :=
  \wcxp(\fml{S};\fml{E})\land\neg\wcxp(\fml{S}\setminus\{i\};\fml{E})
\end{equation}
When addressing duality, different forms of equivalence may or may not
be obtained.
Two FISs $\sv$ and $\sv'$ are \emph{equivalent} if they do not differ
by more than a strictly positive multiplicative constant. Formally,
$\exists(\alpha>0)\forall(\fml{E})\forall(i\in\fml{F}).\sv'(i;\fml{E})=\alpha\times\sv(i;\fml{E})$.
Two FISs $\sv$ and $\sv'$ are \emph{weakly equivalent} if they rank
features in the same order. Formally,
$\forall(\fml{E})\forall(i,j\in\fml{F}).\sv(i;\fml{E})\ge\sv(j;\fml{E})\lequiv\sv'(i;\fml{E})\ge\sv'(j;\fml{E})$.

An FIS $\sv$ is \emph{consistent} with strict duality if $\sv$ is
equivalent to its dual.
An FIS $\sv$ is \emph{weakly consistent} with strict duality if $\sv$
is weakly equivalent to its dual.
Finally, an FIS $\sv$ is \emph{strongly consistent} with strict
duality if $\sv$ is \emph{equal} to its dual.

%\paragraph{P08: Duality.}
%%~\\
%An FIS is said to respect the property of duality if the score remains
%unchanged when (W)AXps in its definition are exchanged with (W)CXps.
%%
%For example, we can define dual FISs by replacing $\cfn{w}$ with the
%following alternative characteristic function:
%\[
%\cfn{w^{d}}(\fml{S};\fml{E}) :=
%\left\{
%\begin{array}{lcl}
%  1 & \quad & \tn{if $\wcxp(\fml{S};\fml{E})$} \\[1.5pt]
%  0 & \quad & \tn{otherwise} \\
%\end{array}
%\right.
%\]
%Furthermore, sets defined using AXps will be replaced by sets defined
%using CXps. As an example, the dual predicate for a critical feature
%becomes,
%\begin{equation} \label{eq:def:criticd}
%  \critic_d(i,\fml{S};\fml{E}) :=
%  \wcxp(\fml{S};\fml{E})\land\neg\wcxp(\fml{S}\setminus\{i\};\fml{E})
%\end{equation}
%%
%Given the above, an FIS defined using $\cfn{w}$ respects duality if
%the computed scores remain unchanged when $\cfn{w}$ is replaced by
%$\cfd{w}$ and the score definitions use CXps instead of AXps.

%% file: nidx.tex
%\section{New Feature Attribution Indices}
\section{Novel \piacronympl}
\label{sec:nidx}

As noted in% %when describing the basic \piacronympl
~\cref{sec:formpi},
%(which were obtained from adapting existing proposals from the domain
%of power indices),
in the case of \piacronympl there exists an underlying
explanation problem and a chosen characteristic function, and  the
optional definition of a chosen power index. In turn, this enables
envisioning additional novel \piacronympl. This section investigates
two examples of novel FISs; others can be envisioned.

\paragraph{FISs based on duality.}
%~\\
%
For each score that is defined using AXps or WAXps, a dual score can
be defined. Evidently, for scores that respect the duality property,
the resulting scores will be the same. However, for others this may
not be the case.
To obtain dual FISs, we start by creating new (dual) template scores,
based on existing template scores, but where the dependency on (W)AXps
is replaced by dependency on (W)CXps.
Afterwards, we instantiate the dual FIS by specifying a dual
characteristic function.
As an example, we consider the responsibility score.
The dual template responsibility score is obtained from the template
responsibility score (see~\eqref{eq:tsr}). Hence, we obtain,
\[
\tsvn{R^d}(i;\fml{E},\cf) :=
\max\left\{\left.\frac{\Delta_i(\fml{S};\fml{E},\cf)}{|\fml{S}|}\,\right|\,\fml{S}\in\mbb{C}_i(\fml{E})\right\}
\]
Now, the resulting FIS is,
\begin{align}
\svn{C}(i;\fml{E}) ~~ := ~~ &
\tsvn{R^d}(i;\fml{E},\cfd{w})
= \nonumber \\
&
\max\left\{\Delta_i(\fml{S};\fml{E},\cfd{w})\vert\fml{S}\in\mbb{C}_i\right\}
\nonumber
\end{align}
%
%\[
%\svn{C}(i;\fml{E}) := \max\left\{\Delta_i(\fml{S};\fml{E},\cfd{w})\vert\fml{S}\in\mbb{C}_i\right\}
%\]
%
This FIS corresponds to the index proposed
%by Chockler et al.
in earlier work~\citep{chockler-acm-tocl08}.

\begin{comment} %% TO INCLUDE IN EXTENDED VERSION
%
\paragraph{FISs based on weighting sets.}
%~\\
%
Let $\omega:2^{\fml{F}}\to\mbb{R}$ be a weight function, decreasing on
set size,
i.e. $\forall(\fml{S},\fml{T}\subseteq\fml{F}).(|\fml{S}|\ge|\fml{T}|)\limply(\omega(\fml{S})\le\omega(\fml{T}))$. Then,
the weighted Holler-Packel FIS is defined by,
%
\[
\svn{H^{\omega}}(i;\fml{E}) :=
\sum\nolimits_{\fml{S}\in\mbb{A}_i(\fml{E})}\left(\sfrac{\omega(\fml{S})}{\mbb{A}(\fml{E})}\right)
\]
%
Observe that setting $\omega(\fml{S})=1$ yields $\svn{H}$, and setting
$\omega(\fml{S})=\sfrac{1}{|\fml{S}|}$ yields $\svn{D}$.
%
\end{comment}

%\paragraph{An FIS not index-based.}
\paragraph{A coverage-based FIS.}
%~\\
%
We illustrate the flexibility of the proposed framework by developing
a new score, which is not derived from an existing power index, and so
not derived from one of the proposed template scores.
Let
$\fml{V}(i;\fml{E})=\{\mbf{x}\in\mbb{F}\,|\,\exists(\fml{S}\in\mbb{A}_i).\fml{S}\subseteq\fml{I}(\mbf{x};\mbf{v})\}$. Clearly,
$\fml{V}$ denotes the points in feature space co\emph{v}ered by some
AXp containing feature $i$. Given $\fml{V}$, the resulting
\emph{coverage} score is defined as follows,
\[
\svn{V}(i;\fml{E}) := \sfrac{|\fml{V}(i;\fml{E})|}{2^{|\fml{F}|}}
\]

%% file: tres.tex
\input{./tabs/sccmp}

%\section{Characterization of Feature Attribution Indices}
\section{Characterization of \piacronympl} \label{sec:tres}

Given the basic \piacronympl studied in~\cref{sec:formpi}, the novel
\piacronympl proposed in~\cref{sec:nidx}, and the properties proposed
in~\cref{sec:nprop}, this section characterizes each of the
\piacronympl in terms of those properties.%
\footnote{%
%Two of the scores, i.e.\
$\svn{S}$ and $\svn{M}$ have also been
studied in recent work~\citep{lhms-corr24a}. However, these were not 
obtained with the generalized formalization this paper proposes.
Also, while the scores $\svn{S}$ and $\svn{M}$ relate with Shapley
values, the relationship with the Shapley-Shubik index was not
discussed. %in~\citep{lhms-corr24a}.
}

\cref{tab:sccmp} summarizes the assessment of the FISs studied in this
paper in terms of existing properties, or the ones also proposed in
this paper.
%summarizes the observations made in this section.
%
Some of the results are proved in~\cref{app:proofs}; the remaining
non-trivial results are proved in the supplemental materials.)
%
%
%(Detailed proofs of all the results are included below and in
%the supplemental materials.)
%
From the table it is apparent that the most interesting FISs are
Shapley-Shubik and Banzhaf, because of the properties that these
scores exhibit. However, the Shapley-Shubik and Banzhaf scores
may require analyzing exponentially many WAXps, whereas other scores,
e.g.\ Deeghan-Packel or Holler-Packel, only require analyzing all the
sets of AXps.
The coverage score, proposed in this paper, represents a possible
middle ground, since it is based on AXps, and it is believed to
exhibit some sort of duality.

%% file: tabs/sccmp.tex
\begin{table*}[t]
  \centering
  \renewcommand{\arraystretch}{1.25}
  \renewcommand{\tabcolsep}{0.5em}
  \begin{tabular}{ccC{3.15cm}C{3.25cm}C{0.55cm}C{0.55cm}C{0.55cm}C{0.55cm}@{}cC{0.55cm}C{0.775cm}C{0.55cm}C{0.55cm}C{0.575cm}}
    % C{0.1pt}%C{0.55cm}c
    \toprule
    %\multirow{2}{*}{}
    %\diagbox{FIS}{Prop.} & Source & P1 & P2 & P3 & P4 & P5 & P6 & P7
    %\\
    \multirow{3}{*}{FIS} &
    \multirow{3}{*}{New?} &
    \multicolumn{2}{c}{Proposal} &
    \multicolumn{10}{c}{Properties}
    %\multirow{2}{*}{\multicolumn{2}{c}{Proposal}} &
    %\multicolumn{}{c}{\multirow{2}{*}{Proposal}} &
    %\multicolumn{9}{c}{Properties} %%&
    \\ \cline{3-14}
    & & \multirow{2}{*}{As score} & \multirow{2}{*}{As index} &
    \multicolumn{5}{c}{On indices} &
    \multicolumn{5}{c}{On scores} 
    \\ \cline{5-8} \cline{10-14}
    & & & &
    %\multirow{4}{*}{On indices} &
    %\multirow{5}{*}{On scores} &
    %%%%\multirow{2}{*}{Uniq.?}
    %\\ \cline{3-13}
    %& & As score & As index
    P01 & P02 & P03 & P04 & & P05 & P06 & P07 & P08 & P09
    \\ \toprule
    $\svn{E}$ & \xmark & \citep{lundberg-nips17} & \citep{shapley-apsr54} & %$\tsvn{S}$
    \cmark & \cmark & \cmark & \cmark & &
    \xmark & \xmark$\,;\:$E & %$\gamma\,{=}\,0$
    \xmark & \xmark & ---
    \\ %
    \midrule
    $\svn{D}$ & \xmark & \citep{izza-aaai24} & \citep{deegan-ijgt78} &
    \xmark & \xmark & \cmark & \cmark & &
    \cmark & \cmark$\,;\:$1 & %$\gamma\,{=}\,1$
    \cmark & \cmark & \xmark
    \\ %\citep{ignatiev-corr23a,izza-aaai24}
    $\svn{H}$ & \xmark & \citep{izza-aaai24} & \citep{holler-je83} &
    \xmark & \xmark & \cmark & \cmark & &
    \cmark & \cmark$\,;\:$M & %$\gamma\,{\not=}\,0$
    \cmark & \cmark & \xmark
    \\ %\citep{ignatiev-corr23a,izza-aaai24}
    $\svn{R}$ & \xmark & \citep{izza-aaai24} & Dual
    of~\citep{chockler-acm-tocl08} &
    \xmark & \xmark & \xmark & \cmark & &
    \cmark & \cmark$\,;\:$H  & %$\gamma\,{\not=}\,0$
    \cmark & \cmark & \xmark
    \\ %\citep{izza-aaai24}
    \midrule
    $\svn{M}$ & \cmark & --- & \citep{shapley-apsr54} & %$\tsvn{S}$
    \cmark & \cmark & \cmark & \cmark & &
    \xmark & \cmark$\,;\:$E  & %$\gamma\,{\not=}\,0$
    \cmark & \xmark & ---
    \\ %\citep{lhms-corr24a}
    $\svn{S}$ & \cmark & --- & \citep{shapley-apsr54} &
    \cmark & \cmark & \cmark & \cmark & &
    \cmark & \cmark$\,;\:$1  & %$\gamma\,{=}\,1$
    \cmark & \cmark & S
    \\ %New, from~\citep{shapley-apsr54} %llhms-corr24a
    $\svn{B}$ & \cmark & --- & \citep{banzhaf-rlr65} &
    \xmark & \cmark & \cmark & \cmark & &
    \cmark & \cmark$\,;\:$M  & %$\gamma\,{\not=}\,0$
    \cmark & \cmark & S
    \\ %New, from~\citep{banzhaf-rlr65}
    $\svn{J}$ & \cmark & --- & \citep{johnston-ep78} &
    \xmark & \cmark & \xmark & \cmark & &
    \cmark & \cmark$\,;\:$E  & %$\gamma\,{\not=}\,0$
    \cmark & \cmark & W?
    \\ %New, from~\citep{johnston-ep78}
    $\svn{A}$ & \cmark & --- & \citep{andjiga-msh03} &
    \xmark & \xmark & \cmark & \cmark & &
    \cmark & \cmark$\,;\:$1  & %$\gamma\,{=}\,1$
    \cmark & \cmark & W?
    \\ %New, from~\citep{andjiga-msh03}
    \midrule
    $\svn{C}$ & \cmark & --- & \citep{chockler-acm-tocl08} &
    \xmark & \xmark & \xmark & \cmark & &
    \cmark & \cmark$\,;\:$H  & %$\gamma\,{\not=}\,0$
    \cmark & \cmark & \xmark
    \\
    %
% TO INCLUDE IN EXTENDED VERSION
%    $\svn{H^{\omega}}$ & \cmark & --- & --- &
%    &  &  &  & &
%    & $\gamma=$ &
%    & &
%    \\
%    %
    $\svn{V}$ & \cmark & --- & NA & 
    \multicolumn{5}{c}{No associated index} &
    %NA & NA & NA & NA & &
    \cmark & \cmark$\,;\:$H & %$\gamma\,{\not=}\,0$
    \cmark & \cmark & W?
    \\
    \bottomrule
  \end{tabular}
  \caption{Characterization of FISs.
    The properties represent: efficiency (P01), symmetry (P02),
    additivity (P03), dummy player (P04), minimal monotonicity (P05),
    $\gamma$-efficiency (P06, where is is shown whether the value of
    $\gamma$ is guaranteed not to be 0, and how difficult it is to
    compute the value of $\gamma$: Easy, Medium, Hard, or its fixed
    value).
    %%where the value of $\gamma$ is shown),
    independence from class labeling (P07), consistency with relevancy
    (P08), and consistency with duality (P09, Plain, Weak or
    Strong). Regarding duality, W? signifies guaranteed not to be
    strong nor plain, but that it could be weak. 
  } \label{tab:sccmp}
\end{table*}

%% file: conc.tex
\section{Conclusions}
\label{sec:conc}

Despite their widespread use,
%of SHAP scores in XAI by feature attribution,
SHAP scores~\citep{lundberg-nips17} can produce
inadequate measures of relative feature importance in XAI by feature
attribution~\citep{msh-cacm24}. This is explained by the characteristic
function used for computing SHAP scores.

\paragraph{Contributions.}
%~\\
%
Motivated by recent work, this paper: i) develops an in-depth analysis
of applying power indices in XAI by feature attribution; ii) proposes
a novel framework for creating new scores measuring feature
importance in XAI; iii) proposes properties that such scores should exhibit;
iv) develops novel scores; and v) characterizes the different scores
in terms of the properties they exhibit.
Furthermore, the paper reveals novel connections between: i) automated
reasoning and its uses in logic-based abduction; ii) game theory and
its applications in a priori voting power and more generally in
characteristic function games; and iii) logic-based explainability and
its usage in machine learning.

%\paragraph{Open questions.}
\paragraph{Research directions.}
%~\\
%
%Besides the listed contributions, the results in the paper also open
Several topics of research can be envisioned.
First, address the complexity of computing
the FISs studied in the paper and other recent
work~\citep{izza-aaai24}. This may enable choosing the most adequate
FIS, given the properties it exhibits and the computational complexity
for its computation.
Second, analyze the many possible FISs that can be envisioned, either
by devising novel characteristic functions, and relating them with the
template scores proposed in this paper, or by proposing new template
scores.
Third, devise new properties, aiming at developing a fine-grained
axiomatic but rigorous characterization of relative feature importance
in explainability.
Fourth, and finally, propose ways for the rigorous approximation of
FISs, namely when analyzing complex ML models.

\jnoteF{Ideas can be used in other settings, e.g. distance-restricted
  AXps, but also constrained AXps}

%A related by somewhat orthogonal link are measures of inconsistency
%for inconsistent theories.

\jnoteF{Future work: complexity of computing difference indices.}

%% file: appendix.tex
\appendix

\section{From FISs back to Power Indices}
\label{ssec:fis2pi}

The derivation of FISs in the case of XAI can be replicated 
in the case of power indices for a priori voting power.
Instead of using WAXps, we consider \emph{winning coalitions} (WCs)
of some weighted voting game $\fml{W}$. A voter $i$ is critical for a 
winning coalition (i.e.\ set of voters voting the same way) $\fml{S}$
if,
\[
\critic(i,\fml{S};\fml{W}) :=
\wc(\fml{S};\fml{W})\land\neg\wc(\fml{S}\setminus\{i\};\fml{W})
\]
(Observe that $\wc$ is a predicate deciding whether the argument is a
winning coalition.)
As a result, we can obtain FISs $\svn{S}$, $\svn{B}$,
$\svn{J}$, $\svn{A}$, and also $\svn{D}$, $\svn{H}$, and
$\svn{R}$, from the respective template scores. (For the last three
FISs above, i.e.\ $\svn{D}$, $\svn{H}$, and $\svn{R}$,
$\mbb{A}$ and $\mbb{A}_i$ represent sets of \emph{minimal} winning
coalitions, either all or just the ones containing voter $i$.)
Also, since there is no associated explanation problem, the FISs in
this concrete setting represent power indices.%
\footnote{The analysis of power indices also serves to illustrate the
flexibility of the generalized formalization proposed in this paper.}

\jnoteF{ToDo}

\section{Proofs for~\cref{tab:sccmp}}
\label{app:proofs}

\begin{proposition}
  The indices $\tsvn{D}$, $\tsvn{H}$, $\tsvn{R}$ and $\tsvn{A}$ do
  not respect symmetry (see P02).
\end{proposition}

\begin{proof}
  The proof uses the following definitions. A set of features
  $\fml{S}\subseteq\fml{F}$ is a \emph{generator} if
  $\forall(i\in\fml{F}\setminus\fml{S}).\waxp(\fml{S}\cup\{i\})$. The
  predicate $\msf{Gen}(\fml{S})$ is true iff $\fml{S}$ is a generator.
  The assumed characteristic function is:
  $\cfn{G}(\fml{S})=\tn{ITE}(\msf{Gen}(\fml{S}),1,0)$.
  Finally, we consider a boolean classifier $\fml{M}_1$ defined on two
  boolean features $\fml{F}=\{1,2\}$, such that $\kappa_1(x_1,x_2)=x_1$.\\
  From P02, an index $\pidx$ is symmetric if, for any $\fml{E}$ and
  symmetric elements $i,j\in\fml{F}$, then $\pidx(i)=\pidx(j)$. In
  particular, if the condition is false for $\cfn{G}$ and $\kappa_1$,
  then the index is not symmetric.\\
  For the chosen example, ther are only 2 features, and so if
  $\cfn{G}(1)=\cfn{G}(2)$, then
  $\pidx(1;\fml{E},\cfn{G})=\pidx(2;\fml{E},\cfn{G})$. Since by
  construction, $\cfn{G}(1)=\cfn{G}(2)$, then it must be the case that 
  $\pidx(1;\fml{E},\cfn{G})=\pidx(2;\fml{E},\cfn{G})$.\\
  We now consider the template scores (which correspond to indices)
  $\tsvn{D}$, $\tsvn{H}$, $\tsvn{R}$ and $\tsvn{A}$.
  It is plain to compute
$\tsvn{D}(1;\fml{E},\cfn{G})=\tsvn{H}(1;\fml{E},\cfn{G})=\tsvn{R}(1;\fml{E},\cfn{G})=1$,
  $\tsvn{A}(1;\fml{E},\cfn{G})=\sfrac{1}{2}$, whereas
$\tsvn{D}(2;\fml{E},\cfn{G})=\tsvn{H}(2;\fml{E},\cfn{G})=\tsvn{R}(2;\fml{E},\cfn{G})=\tsvn{A}(2;\fml{E},\cfn{G})=0$.
  This proves the result.
\end{proof}

\begin{proposition}
  $\svn{D}$, $\svn{H}$ and $\svn{R}$ do not respect weak duality (see
  P09), and so do not respect neither plain nor strong duality.
  Furthermore, $\svn{J}$, $\svn{A}$ and $\svn{V}$ do not respect
  neither plain nor strong duality.
\end{proposition}

\begin{proof}
  %Let's take the example is κ(x1,x2,x3,x4):=x1⋀(x2⋁x3⋀x4) so
  %mathbb{A}={{1,2},{1,3,4}} and mathbb{C}={{1},{2,3},{2,4}}.
  %We can use it to compute:
  Let $\kappa(x_1,x_2,x_3,x_4)=x_1\land(x_2\lor{x_3}\land{x_4})$, with
  $\mbb{A}=\{\{1,2\},\{1,3,4\}\}$ and
  $\mbb{C}=\{\{1\},\{2,3\},\{2,4\}\}$. We consider the dual FISs
  obtained from $\svn{D}$, $\svn{H}$, and $\svn{R}$, by replacing AXps
  with CXps. Let $\svd{D}$, $\svd{H}$, and $\svd{R}$ be the resulting
  dual FISes.
  Thus, for $\svn{D}$, we compute,
  %wFFA_a(1)=5/12, wFFA_a(2)=1/4, wFFA_a(3)=1/6 and wFFA_a(4)=1/6.
  %wFFA_c(1)=1/3, wFFA_c(2)=1/3, wFFA_c(3)=1/6 and wFFA_c(4)=1/6.
  \begin{align}
    &
    \svn{D}(1)=\sfrac{5}{12},
    \svn{D}(2)=\sfrac{1}{4},
    \svn{D}(3)=\sfrac{1}{6},
    \svn{D}(4)=\sfrac{1}{6}
    \nonumber \\
    &
    \svd{D}(1)=\sfrac{1}{3},
    \svd{D}(2)=\sfrac{1}{3},
    \svd{D}(3)=\sfrac{1}{6},
    \svd{D}(4)=\sfrac{1}{6}
    \nonumber %\\
  \end{align}
  which proves that $\svn{D}$ does not respect weak duality.
  Similarly, for $\svn{H}$ we compute,
  %
  %FFA_a(1)=1, FFA_a(2)=1/2, FFA_a(3)=1/2 and FFA_a(4)=1/2.
  %FFA_c(1)=1/3, FFA_c(2)=2/3, FFA_c(3)=1/3 and FFA_c(4)=1/3.
  \begin{align}
    &
    \svn{H}(1)=1,
    \svn{H}(2)=\sfrac{1}{2},
    \svn{H}(3)=\sfrac{1}{2},
    \svn{H}(4)=\sfrac{1}{2}
    \nonumber \\
    &
    \svd{H}(1)=\sfrac{1}{3},
    \svd{H}(2)=\sfrac{2}{3},
    \svd{H}(3)=\sfrac{1}{3},
    \svd{H}(4)=\sfrac{1}{3}
    \nonumber %\\
  \end{align}
  which proves that $\svn{H}$ does not respect weak duality.
  Moreover, for $\svn{R}$, we compute,
  %
  %mFFA_a(1)=1/4, mFFA_a(2)=1/4, mFFA_a(3)=1/6 and mFFA_a(4)=1/6.
  %mFFA_c(1)=1/3, mFFA_c(2)=1/6, mFFA_c(3)=1/6 and mFFA_c(4)=1/6.
  \begin{align}
    &
    \svn{R}(1)=\sfrac{1}{4},
    \svn{R}(2)=\sfrac{1}{4},
    \svn{R}(3)=\sfrac{1}{6},
    \svn{R}(4)=\sfrac{1}{6}
    \nonumber \\
    &
    \svd{R}(1)=\sfrac{1}{3},
    \svd{R}(2)=\sfrac{1}{6},
    \svd{R}(3)=\sfrac{1}{6},
    \svd{R}(4)=\sfrac{1}{6}
    \nonumber %\\
  \end{align}
  which proves that $\svn{R}$ does not respect weak duality.\\
  The proof regarding $\svn{J}$, $\svn{A}$ and $\svn{V}$ is included
  in the supplemental materials.
\end{proof}

\begin{proposition}
  $\svn{S}$ and $\svn{B}$ respect strong duality (see P09).
\end{proposition}

\begin{proof}
  We consider first the strong duality of $\svn{S}$. %\\
  By definition, it is the case that
  $\forall(\fml{S}\subseteq\fml{F}).\wcxp(\fml{F}\setminus\fml{S})\leftrightarrow\neg\waxp(\fml{S})$.
  Also, recall that $\cfn{w}(\fml{S};\fml{E}):=\waxp(\fml{S};\fml{E})$.
  Thus,
  %$\forall(i\in\fml{F}),\forall(\fml{S}\subseteq(\fml{F}\setminus\{i\}))$,
  $\forall(i\in\fml{F}),\forall(\fml{S}\in\{\fml{T}\subseteq\fml{F}\,|\,i\in\fml{T}\})$ we have,
  \begin{align}
    & \Delta_i(\fml{S};\fml{E},\cfn{w}) =1
    \nonumber\\[3pt]
    % Expand
    \Leftrightarrow\:&
    \neg\waxp(\fml{S}\setminus\{i\};\fml{E})\land\waxp(\fml{S};\fml{E})
    \nonumber\\[3pt]
    % Replace
    \Leftrightarrow\:&
    \wcxp(\fml{F}\setminus(\fml{S}\setminus\{i\}))\land\neg\wcxp(\fml{F}\setminus\fml{S})
    \nonumber\\[3pt]
    \Leftrightarrow\:&
    \wcxp((\fml{F}\setminus\fml{S})\cup\{i\}\land\neg\wcxp(\fml{F}\setminus\fml{S})
    \nonumber\\
    \Leftrightarrow\:&
    \Delta_i((\fml{F}\setminus\fml{S})\cup\{i\};\fml{E},\cfd{w})=1
    \nonumber
  \end{align}
  Now, let
  $\Phi(i):=%
  \{\fml{S}\in\{\fml{T}\subseteq\fml{F}\,|\,i\in\fml{T}\}%
  \,|\,\Delta_{i}(\fml{S};\fml{E},\cfn{w})=1\}$.
  Then, by construction,
  $\svn{S}(i)=\sum_{\fml{S}\in\Phi(i)}\varsigma(|\fml{S}|)$ (because
  $\Delta_{i}(\fml{S};\fml{E},\cfn{w})=0$ otherwise) and, by the
  equivalence above,
  $\svd{S}(i)=\sum_{\fml{S}\in\Phi(i)}\varsigma(|(\fml{F}\setminus\fml{S})\cup\{i\}|)$.
  However,
  it is immediate to prove that 
  $\varsigma(|\fml{S}|)=\varsigma(|(\fml{F}\setminus\fml{S})\cup\{i\}|)$,
  and so the two sums are also equal,
  i.e.\ $\svn{S}(i;\fml{E},\cfn{w})=\svd{S}(i;\fml{E},\cfd{w})$. \\
  For $\svn{B}$, a similar argument is used. This concludes the proof.
  \qedhere
\end{proof}

\begin{proposition} % P09
  For $\svn{J}$, $\svn{A}$ and $\svn{V}$, and for propoerty P09,
  neither plain nor strong duality hold.
\end{proposition}

\begin{proof}
  The counterexample used in the proof of Proposition~2 (see paper)
  can also be used for the claims above.
  Concretely, for $\svn{J}$ we get,
  \begin{align}
    & \svn{J}(1)=\sfrac{17}{6}, \svn{J}(2)=\sfrac{3}{2},
    \svn{J}(3)=\svn{J}(4)=\sfrac{1}{3}
    \nonumber \\
    & \svd{J}(1)=5, \svd{J}(2)=2, \svd{J}(3)=\svd{J}(4)=\sfrac{1}{2}
    \nonumber
  \end{align}
  For $\svn{A}$ we get,
  \begin{align}
    & \svn{A}(1)=\sfrac{7}{20}, \svn{A}(2)=\sfrac{7}{30},
    \svn{A}(3)=\svn{A}(4)=\sfrac{1}{15}
    \nonumber \\
    & \svd{A}(1)=\sfrac{17}{66}, \svd{A}(2)=\sfrac{4}{33},
    \svd{A}(3)=\svd{A}(4)=\sfrac{1}{22}
    \nonumber
  \end{align}
  Finally, for $\svn{V}$ we get,
  \begin{align}
    & \svn{V}(1)=\sfrac{5}{16}, \svn{V}(2)=\sfrac{1}{4},
    \svn{V}(3)=\svn{V}(4)=\sfrac{1}{8}
    \nonumber \\
    & \svd{V}(1)=\sfrac{1}{2}, \svd{V}(2)=\sfrac{3}{8},
    \svd{V}(3)=\svd{V}(4)=\sfrac{1}{4}
    \nonumber
  \end{align}
  Which proves the result.\\
  Unlike $\svn{D}$, $\svn{H}$, $\svn{R}$, the order remains unchanged
  and so this does not allow discarding Weak duality.
\end{proof}

\begin{proposition} % #1
  P01, P02, P03 and P04 hold for $\tsvn{S}$.
\end{proposition}

\begin{proof}
  This is a consequence of the uniqueness of the Shapley
  value~\cite{shapley-ctg53}.
\end{proof}

\begin{proposition} % #2
  For the proposed FISs, and for P07 and P08 the properties that hold
  (resp.\ not hold) are as shown in Table~1.
\end{proposition}

\begin{proof}
  The entries for $\svn{E}$, $\svn{M}$ and $\svn{S}$ result from the
  results in~\citep{lhms-corr24a}.\\
  Furthermore, the proofs from~\citep{lhms-corr24a} (see propositions
  5.2 and 5.3) also generalize
  for $\svn{D}$, $\svn{H}$, $\svn{R}$, $\svn{B}$, $\svn{J}$,
  $\svn{A}$, since the proofs only depend on the characteristic
  functions used, and the same characteristic function can be used for
  all these FISs.\\
  For $\svn{V}$, it is immediate that a definition based on AXps does
  not depend on the values taken by the classifier. Also, if a feature
  does not occur in AXps, then its score will be 0.
\end{proof}

\begin{proposition} % #3, #4
  P01 does not hold for $\tsvn{D}$, $\tsvn{H}$, $\tsvn{R}$,
  $\tsvn{B}$, $\tsvn{J}$, and $\tsvn{A}$. 
\end{proposition}

\begin{proof}
  For $\tsvn{D}$, $\tsvn{H}$, and $\tsvn{R}$, the value of efficiency
  is computed in~\citep{ignatiev-corr23a}, and it is different from
  what P01 would imply.\\
  For $\tsvn{B}$, $\tsvn{J}$, and $\tsvn{A}$, if P01 would hold, then
  that would falsify Shapley's theorem~\citep{shapley-ctg53}.
\end{proof}

\begin{proposition} % #5
  P02 holds for $\tsvn{B}$ and $\tsvn{J}$.
\end{proposition}

\begin{proof}
 The sums for two symmetric features $i$ and $j$ both use all sets (so
 the same sets), have $\Delta_i=\Delta_j$ (due to symmetry) and the
 same coefficient and so are equal.
\end{proof}

\begin{proposition} % #6
  P03 holds for $\tsvn{D}$, $\tsvn{H}$, $\tsvn{B}$ and $\tsvn{A}$.
\end{proposition}

\begin{proof}
  By construction, the sum operator and the $\Delta_i$ are linear.
  As the multiplicative constants do not depend on the characteristic
  function, the score is linear.
\end{proof}

\begin{proposition} % #7
  P03 does not hold for $\tsvn{J}$.
\end{proposition}

\begin{proof}
  For $\tsvn{J}$, and unlike the previous proof, the multiplicative
  constant depends on the characteristic function and so break
  linearity.
\end{proof}

\begin{proposition} % #8
  P03 does not hold for $\tsvn{R}$.
\end{proposition}

\begin{proof}
  The $\max$ operator is not linear; it is easy to break the linearity
  by taking $\kappa(x_1)=x_1$ with $\cf_1=\cf_2$ as the characteristic
  function used to $\tsvn{R}$. Hence, lack of linearity becomes
  apparent.
  %linearity become a 2=1 condition.
\end{proof}

\begin{proposition} % #9
  P04 holds for $\tsvn{D}$, $\tsvn{H}$, $\tsvn{R}$, $\tsvn{B}$,
  $\tsvn{J}$ and $\tsvn{A}$.
\end{proposition}

\begin{proof}
  The summations are trivially 0 by construction when
  $\Delta_i(\fml{S})=0$ for every set $\fml{S}\subseteq\fml{F}$.
\end{proof}

\begin{proposition} % #10
  P05 holds for $\svn{D}$, $\svn{H}$, $\svn{R}$, $\svn{S}$, $\svn{B}$,
  $\svn{J}$, $\svn{A}$, $\svn{C}$, and $\svn{V}$, but not for
  $\svn{E}$ and $\svn{M}$.
\end{proposition}

\begin{proof}(Sketch)
  The proof for each FIS is straightforward and only requires
  comparing the terms for each set $\fml{S}$ and apply the axiom.
  What remains becomes fairly similar to the arguments used in the
  case of P08.
\end{proof}

\begin{proposition}
  For P06, the results in Table~1 hold.
\end{proposition}

\begin{proof}(Sketch)
  The formula comes from P01 and so is easy to compute.
  For $\svn{S}$, the simple valued is proved
  in~\citep{lhms-corr24a}.\\
  For the remaining FISs,
  it suffices to manipulate the sum that defines $\gamma$, aggregating
  the two sums and simplifying as much as possible.
\end{proof}

\begin{comment}
%
%
\begin{proposition}
  For the proposed FISs, and for P01, P02, P03 and P04, the properties
  that hold (resp.\ not hold) are as shown in~Table~1.
\end{proposition}

\begin{proof}(Sketch)
  The properties of power indices, but also of payoff vectors are
  well-known. These results are a consequence of those properties.
\end{proof}

\begin{proposition}
  For the proposed FISs, and for P05, the properties that hold
  (resp.\ not hold) are as shown in~Table~1.
\end{proposition}

\begin{proof}(Sketch)
    For $\svn{D}$, $\svn{H}$, $\svn{R}$, the result adapted
    from~\citep{izza-aaai24}. For the remaining FISs for which P05 
    holds, the same arguments can be used.
\end{proof}

\begin{proposition}
  For the proposed FISs, and for P07 and P08, the properties that hold
  (resp.\ not hold) are as shown in~Table~1.
\end{proposition}

\begin{proof}(Sketch)
  The examples studied in earlier work highlight why both P07 and P08
  fail to hold for SHAP scores and its variants.
  Furthermore, by being defined in terms of AXps/CXps, independence
  from class labeling becomes apparent. Since irrelevant features do
  not occur in AXps/CXps, then consistency with relevancy must hold
  for all the FISs defined using AXps/CXps.
\end{proof}
%
%
\end{comment}

%% file: replbib.tex
% RequiredL: \usepackage{etoolbox}
%\providetoggle{mkbbl}
\newtoggle{mkbbl}
% Contents if using bibtex: "\settoggle{mkbbl}{true}"
% Contents if inputing pre-generated file: "\settoggle{mkbbl}{false}"

%% file: togbbl.tex
\settoggle{mkbbl}{false}